\documentclass[letterpaper, 10 pt, conference, final]{ieeeconf}
\IEEEoverridecommandlockouts
\pdfoutput=1

\let\proof\relax
\let\endproof\relax
\usepackage{amsthm}
\usepackage{graphics}
\usepackage{amsmath}
\usepackage{amsfonts}
\usepackage{graphicx}
\usepackage{xspace}
\usepackage[dvipsnames]{xcolor}
\usepackage[]{subcaption}
\let\labelindent\relax
\usepackage{enumitem}
\usepackage{algorithm}
\usepackage{algorithmicx}
\usepackage[noend]{algpseudocode}
\algrenewcommand\algorithmicindent{0.75em}
\usepackage{caption}
\usepackage{comment}
\usepackage{amssymb}
\usepackage[normalem]{ulem}
\usepackage[noadjust]{cite}
\usepackage{hyperref}
\usepackage{wrapfig}
\usepackage{url}
\newcommand{\ignore}[1]{}

 \def\C{\mathcal{C}} 
  
  \def\E{\mathcal{E}}
  
 \def\T{\mathcal{T}} 
 \def\V{\mathcal{V}} 
 \def\W{\mathcal{W}} 
\def\M{\mathcal{M}} \def\X{\mathcal{X}} \def\A{\mathcal{A}}

\def\x{\mathbf{x}}

\def\SE{\mathcal{SE}}
\def\SO{\mathcal{SO}}

\newcommand{\Cpp}{C\raise.08ex\hbox{\tt ++}\xspace}

\newtheorem{lem}{Lemma}
\newtheorem{thm}{Theorem}

\theoremstyle{definition}
\newtheorem{dft}{Definition}
\newtheorem{prob}{Problem}

\newenvironment{proofsketch}{%
  \proof}{\endproof}

\newcommand\algname[1]{\textsf{#1}\xspace}

\newcommand\astar{\algname{A*}}
\newcommand\rrt{\algname{RRT}}
\newcommand\rrts{\algname{RRTs}}
\newcommand\rrtstar{\algname{RRT*}}

\newcommand\rcs{\algname{RCS}}

\newcommand\ros{\algname{RCS*}}
\newcommand\rosnr{\algname{RCS*\_NP}}

\newcommand\aft{\algname{AFT}}
\newcommand\aorrt{\algname{AO-RRT}}

\colorlet{pink}{red!40}
\colorlet{light_blue}{cyan!60}

\newboolean{POI}
\newboolean{ROI}

\setboolean{POI}{true}
\ifthenelse{\boolean{POI}}
	{\setboolean{ROI}{false}}
	{\setboolean{ROI}{true}}

\newcommand{\inspectionType}[2]
{\ifthenelse{\boolean{POI} }{{#1}}{}\ifthenelse{\boolean{ROI}}{#2}{}}

\title{\LARGE \bf
Resolution-Optimal Motion Planning for Steerable Needles
}

\author{Mengyu Fu$^{1}$,%
\thanks{
This project was supported in part
by the United States National Institutes of Health (NIH) under award R01EB024864;
the United States National Science Foundation (NSF) under awards 2008475 and 2038855; 
the Israeli Ministry of Science, Technology and Space (MOST) under awards 3-17385 and 3-16079; 
and 
the United States-Israel Binational Science Foundation (BSF) under award 2019703.
}
\thanks{
Code is available at~\cite{Fu2022_GitHub}.
}
\thanks{$^{1}$M.\ Fu and R.\ Alterovitz are with the Department of Computer Science, University of North Carolina at Chapel Hill, Chapel Hill, NC 27599, USA.
        {\tt\footnotesize \{mfu,ron\}@cs.unc.edu}}
Kiril Solovey$^{2}$,%
\thanks{$^{2}$K.\ Solovey and O.\ Salzman are with Computer Science Department, Technion - Israel Institute of Technology, Israel.
        {\tt\footnotesize kirilsol@stanford.edu, osalzman@cs.technion.ac.il}}
Oren Salzman$^{2}$, %
and Ron Alterovitz$^{1}$%
}

\begin{document}

\maketitle
\thispagestyle{empty}
\pagestyle{empty}

\begin{abstract}

Medical steerable needles can follow 3D curvilinear trajectories inside body tissue, enabling them to move around critical anatomical structures and precisely reach clinically significant targets in a minimally invasive way.
Automating needle steering, with motion planning as a key component, has the potential to maximize the accuracy, precision, speed, and safety of steerable needle procedures.
In this paper, we introduce the first resolution-optimal motion planner for steerable needles that offers excellent practical performance in terms of runtime while simultaneously providing strong theoretical guarantees on completeness and the global optimality of the motion plan in finite time.
Compared to state-of-the-art steerable needle motion planners, simulation experiments on realistic scenarios of lung biopsy demonstrate that our proposed planner is faster in generating higher-quality plans while incorporating clinically relevant cost functions.
This indicates that the theoretical guarantees of the proposed planner have a practical impact on the motion plan quality, which is valuable for computing motion plans that minimize patient trauma.

\end{abstract}

\section{Introduction}
\label{sec:intro}

\begin{figure}
    \centering
    \includegraphics[width=\linewidth]{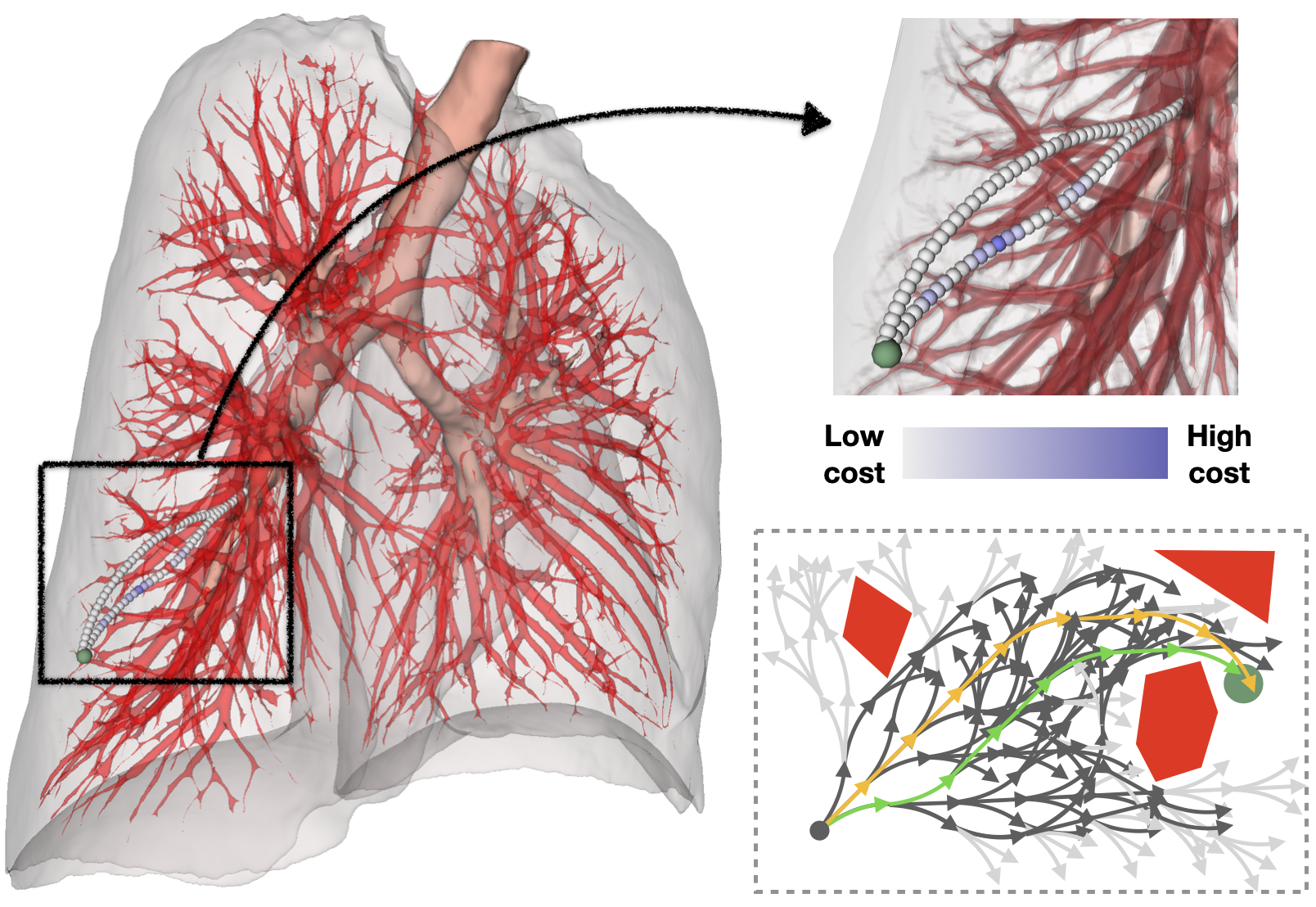}
    \caption{
    \textbf{Left:} 
    Overview of two different steerable-needle motion plans, both reaching a nodule (green) in the lung parenchyma for biopsy or cancer treatment while avoiding critical anatomical structures such as the bronchial tubes (brown) and major blood vessels (red).
    \textbf{Top right:}
    A zoomed-in view of two different plans where small blood vessels are rendered in grayscale.
    We use the method in~\cite{Fu2018_IROS} to reconstruct a cost map that represents the risk of puncturing small blood vessels.
    Different colors along the plans show different costs on the cost map.
    The top plan is computed with our proposed planner, \ros.
    The bottom plan is computed with our previous \rcs algorithm~\cite{Fu2021_RSS} and has a higher cost.
    \textbf{Bottom right:}
    A 2D illustration of the tree grown using our resolution-optimal motion planner towards a goal region (green) while avoiding obstacles (red).
    The best (shortest) plan found is shown in light green, and another valid but worse (longer) plan is shown in yellow.
    With a combination of generic and domain-specific optimizations, we can greatly shrink the search space (i.e., discard the light gray edges).
    }
    \label{fig:cover}
\end{figure}

Medical steerable needles have the potential to improve patient care in diagnostic and therapeutic procedures including biopsy, localized drug delivery, and radioactive seed implantation for cancer treatment~\cite{Abolhassani2007_MEP}.
Steerable needles have a small diameter and are made of a highly flexible material, which allows them to follow 3D curvilinear trajectories inside the tissue. 
These properties enable steerable needles to move around critical anatomical structures to reduce patient trauma and reach sites previously unreachable with traditional straight needles~\cite{Alterovitz2005_ICRA,Cowan2011_Chapter,Park2005_ICRA,Webster2006_IJRR}.

Automating needle steering can improve the accuracy, precision, speed, and safety of steerable-needle procedures. 
Automating these procedures can also facilitate their broad use, since manual control of a steerable needle is challenging due to the nonholonomic constraints on the needle's motion and the high level of precision required to operate it. 
A key component of automating steerable needle procedures is motion planning: computing feasible, obstacle-avoiding trajectories through the tissue to reach a target. 
The trajectory of the needle through tissue should also maximize patient safety, which can be quantified using metrics such as minimizing trajectory length~\cite{Favaro2018_ICRA}, maximizing clearance from obstacles~\cite{Wein2008_IJRR, Agarwal2018_TALG, Kuntz2015_IROS, Strub2021_arXiv}, and minimizing damage to sensitive tissue~\cite{Fu2018_IROS,BRRSK21}.
An example is shown in Fig.~\ref{fig:cover}.

A motion planner for steerable needles should ideally offer guarantees on both finite-time completeness (i.e., compute in finite time a motion plan or indicate that none exists) and optimality (i.e., return a globally optimal motion plan with respect to a chosen cost metric).
Most prior motion planners for steerable needles lack one or both of these criteria.
For instance, some methods for steerable needle motion planning lack completeness guarantees~\cite{Duindam2010_IJRR, Favaro2018_ICRA, Hauser2009_RSS, Patil2014_TRO, Seiler2012_IJRR,Van2010_WAFR,Liu2016_RAL}, and so may fail to find a motion plan when one exists.
Some methods do aim to optimize motion plan cost but they lack \emph{global} optimality guarantees~\cite{Liu2016_RAL,Pinzi2019_IJCARS,Favaro2018_ICRA}.
 
Some sampling-based planners are known to be both complete and optimal, albeit those properties are usually proven only for an asymptotic regime where the number of samples tends to infinity~\cite{Lavalle.Kuffner.2001,Kleinbort2018_RAL,Hauser2016_TRO,Kleinbort2020_ICRA,Karaman2011_IJRR,SoloveyEA2020,Li2016_IJRR,Sun2015_TRO,SH15a}.
Thus, it is unclear what should be the number of samples necessary to achieve those guarantees in practice.
Recent work has developed optimality guarantees for finite sampling, although those results cannot be currently applied to steerable needles as they deal with holonomic systems~\cite{Tsao.ea.2020,DayanSoloveyETAL2021}.

In this paper, we introduce the first motion planner for steerable needles that offers excellent practical performance in terms of runtime while providing strong theoretical guarantees on completeness and the cost of the motion plan in \emph{finite time}. 
In particular,  we consider a specific type of optimality in relation to the motion plan cost---resolution optimality.
Generally speaking, a resolution characterizes the discretization of some space (e.g., state space, configuration space, action space, and time). An algorithm is \textit{resolution complete} if there exists a fine-enough resolution with which the algorithm finds a motion plan in finite time when a qualified motion plan exists, and otherwise correctly returns that no such plan exists~\cite{Fu2021_RSS}.
An algorithm is \textit{resolution optimal} if it is resolution complete and if, when it does return a motion plan, the plan’s cost is guaranteed to be within a desired approximation factor of the cost of a globally optimal qualified motion plan.

Our new motion planner builds on Resolution-Complete Search (\rcs)~\cite{Fu2021_RSS}, which is a resolution-complete but not resolution-optimal motion planner for steerable needles.
If a motion plan exists, \rcs would find a motion plan in finite time assuming that the parameter resolution is fine enough, but it provides no guarantees on the motion plan cost.
To achieve resolution optimality, we enhance \rcs, which explores the needle's state space in an \astar-like fashion, with cost-aware duplicate pruning
while incorporating motion plan cost tracking and a heuristic function to improve efficiency.
We also provide a proof sketch to show the resolution optimality of our method with a careful discussion of assumptions and required conditions.
We refer to our new method as \ros, a resolution-optimal version of \rcs.
We also demonstrate experimentally on a realistic lung biopsy scenario that \ros outperforms the state-of-the-art in terms of runtime and plan quality.

\section{Related Work}
\label{sec:related_work}

\subsection{Motion planning for steerable needles}
A variety of approaches have been proposed for the motion planning of steerable needles. 
Duindam et al.~\cite{Duindam2010_IJRR} proposed a planner based on inverse kinematics but provided no theoretical guarantees.
Liu et al.~\cite{Liu2016_RAL} developed the Adaptive Fractal Tree (\aft) for needle steering.
Their method iteratively refines the lowest-cost plan from the previous iteration, but refining the best plan of a coarse resolution does not necessarily lead to the best plan in a finer resolution.
Pinzi et al.~\cite{Pinzi2019_IJCARS} extended it to account for goal orientation constraints.

Some planners adapt sampling-based methods such as
Rapidly-exploring Random Tree (\rrt)~\cite{Lavalle.Kuffner.2001} for steerable needles.
Xu et al.~\cite{Xu2008_ICASE} used an \rrt variant for needle steering but showed low time efficiency.
Patil et al.~\cite{Patil2014_TRO} developed an \rrt-based needle planner that samples in the 3D workspace rather than the configuration space.
Sampling in a lower-dimension space and their customized distance function made the planner work efficiently in many practical cases, but this also invalidates the probabilistic-completeness guarantee of \rrt~\cite{Lavalle.Kuffner.2001,Kleinbort2020_ICRA}.
To avoid dealing with curvature constraints directly, Favaro et al.~\cite{Favaro2018_ICRA} proposed a hybrid method to combine sampling and smoothing.
First, a tree embedded in the 3D workspace is built with \rrtstar~\cite{Karaman2011_IJRR}, then candidate plans found by the tree are smoothed to further account for the curvature constraint.
However, such a decoupling invalidates the asymptotic optimality guarantee~\cite{Karaman2011_IJRR,SoloveyEA2020}.
Sun et al.~\cite{Sun2015_TRO} proposed a needle planner by building multiple \rrts, which is asymptotically optimal when the number of trees tends to infinity.
Other methods focus on accounting for the uncertainty during needle insertion without providing formal guarantees~\cite{Hauser2009_RSS,Van2010_WAFR,Seiler2012_IJRR}.

\subsection{Resolution-optimal motion planners}
Although resolution completeness has been frequently mentioned~\cite{Barraquand1991_IJRR,Barraquand1993_Algorithmica,Cheng2002_ICRA, Lindemann2006_ICRA,Yershov2010_WAFR}, resolution optimality earned little attention, possibly due to being rather complex to analyze mathematically, particularly for nonholonomic systems.
Consequently, many planners developed for nonholonomic systems focus on asymptotic optimality instead~\cite{Gammell2021_AR,Hauser2016_TRO,Li2016_IJRR,Shome2021_ICRA}.

Barraquand et al.~\cite{Barraquand1993_Algorithmica} proposed a resolution-complete planner for single or multiple robots with nonholonomic constraints.
Their method is also optimal with respect to the number of reverse maneuvers in the plan. 
Pivtoraiko et al.~\cite{Pivtoraiko2009_JFR} proposed the idea of motion planning using state lattices for field robots.
Their state lattices planner is resolution optimal since the search is optimal for a graph of some resolution and the discrete state grid approximates the continuous space as resolution increases.
Ljungqvist et al.~\cite{Ljungqvist2017_IVS} later extended~\cite{Pivtoraiko2009_JFR} for a general two-trailer system in 2D.
However, these methods are designed for large-scale workspaces, making them unsuitable for tasks where a high level of precision is required, such as for steerable needles.

\section{Problem Definition}
\label{sec:pdef}

We consider bevel-tip flexible steerable needles~\cite{Alterovitz2005_ICRA,Cowan2011_Chapter,Park2005_ICRA,Webster2006_IJRR} controlled by insertion and rotation at their base.
A bevel-tip steerable needle is made of a highly flexible material such that, when being inserted through tissue, asymmetric forces applied by the bevel cause the needle to take a curved trajectory.
The maximum curvature of the needle's path in the tissue, $\kappa_{\max}$, is influenced by the mechanical design of the needle and the tissue that the needle moves through.
Additionally, axially rotating the needle at its base changes the direction the bevel is 
\begin{wrapfigure}{r}{0.25\textwidth}
\vspace{-4mm}
  \begin{center}
    \includegraphics[width=0.23\textwidth]{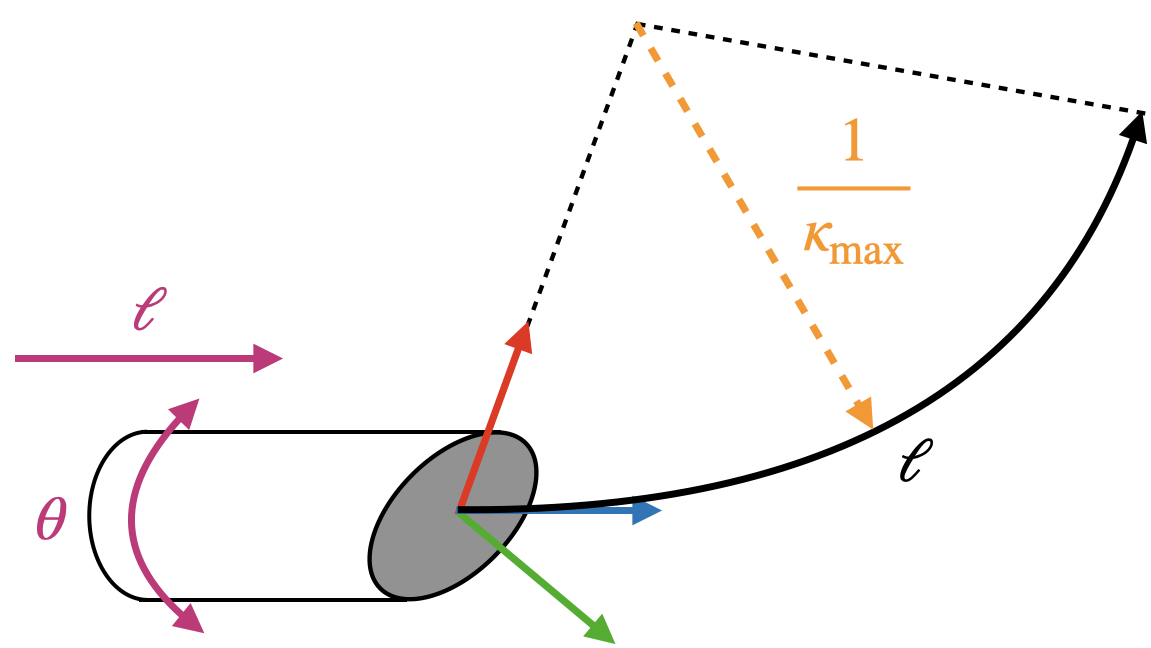}
  \end{center}
\label{fig:needle-control}
\vspace{-2mm}
\end{wrapfigure}
facing, enabling us to control the steering direction.
The figure to the right illustrates the needle's kinematics, where the needle can be inserted by length~$\ell$ and axially rotated at its base  by~$\theta$.

We make the common assumption that the steerable needle is sufficiently flexible so the needle shaft moves along the trajectory created by its tip while the lateral motions are negligible.
Thus, the configuration space of the steerable needle is defined by the pose of its tip, $\X \subset \SE(3)$, where a given configuration $\mathbf{x}=(p, q)\in \X$, specifies the position and orientation components $p\in \mathbb{R}^3$ and $q\in \SO(3)$, respectively. We assume that $\X$ is compact.
We denote the 3D workspace by $\W \in \mathbb{R}^3$, a subset of which is occupied by obstacles~$\W_{\rm obs} \subset \W$.
A configuration~$\mathbf{x} = (p, q)$ is \textit{collision free} if and only if
$p \not \in \W_{\rm obs}$.
We define $\X_{\rm free}$ as the union of all collision-free configurations.

A motion plan of the needle is a trajectory
$\sigma: [0, \ell] \rightarrow \X$,
where $\ell$ is the length of the trajectory. We also use the notation $\ell_\sigma$ to denote the length of $\sigma$. 
A motion plan (or trajectory) $\sigma$ is \textit{collision free} if
$\forall s \in [0, \ell], \sigma(s) \in \X_{\rm free}$.
To evaluate the quality of a motion plan, we consider a configuration-based cost function
$c: \X \rightarrow \mathbb{R}$.
We require $c$  to be well behaved (formally defined in Sec.~\ref{sec:theory}), which includes being Lipshitz continuous and bounded within $[c_{\min}, c_{\max}]$.
We define the cost of a motion plan as the integral of the configuration-based cost along a given trajectory $\sigma$, i.e., 
\begin{equation}
    \C(\sigma) = \int_{0}^{\ell} {c\big(\sigma(s)\big)\mathrm{d}s}.
\label{eq:cost}
\end{equation}
This definition captures a variety of cost functions, including trajectory length and integrating over a cost map derived from medical images.

A motion plan is (kinematically) \textit{feasible} if the curvatures along the trajectory never exceed~$\kappa_{\max}$. A motion plan is \textit{valid} if it is collision free and feasible.
We are now ready to state the steerable needle motion planning problem.

\begin{prob}
\label{prob:MP}
An optimal steerable needle motion planning problem is defined as the tuple
$\Delta = (\X, \W_{\rm obs}, \mathbf{x}_{\rm start}, p_{\rm goal}, \tau, \ell_{\max}, \kappa_{\max}, \C)$,
where~$\W_{\rm obs}$ is the obstacle set,
$\mathbf{x}_{\rm start}$ is the start configuration,
$p_{\rm goal} \in \W$ is the goal point,
$\tau > 0 $ is the goal tolerance,
$\ell_{\max}$ is the maximum insertion length,
$\kappa_{\max}$ is the maximum curvature,
and $\C$ is a cost function.
The problem calls for computing an optimal valid motion plan
$\sigma^* = {\rm argmin}_{\sigma} \C(\sigma)$
subject to:
\begin{enumerate}
    \item [] $\sigma$ is valid,
    \item [] $\sigma(0) = \mathbf{x}_{\rm start}$,
    \item [] $\ell_\sigma \leq \ell_{\max}$,
    \item [] $\|{\rm Proj}(\sigma(\ell_\sigma)) - p_{\rm goal}\|_2 \leq \tau$,
\end{enumerate}
where ${\rm Proj}(\mathbf{x}) = p$ for $\mathbf{x} = (p, q)$.
\end{prob}

As we show in Sec.~\ref{sec:theory}, for any given instance of Problem~\ref{prob:MP}, under some mild assumptions, there exists a fine-enough cutoff resolution~$R_{\min} = \{\delta\ell_{\min}, \delta\theta_{\min}\}$ 
(corresponding to the needle's insertion and axial rotation, respectively)
for which our planner is guaranteed to return a motion plan with
a cost to be within a desired approximation factor of a globally optimal qualified motion plan in finite time (if any qualified motion plan exists), or indicate that no qualified motion plan exists.
Similar to~\cite{Fu2021_RSS}, we assume there exist minimal motions that are precisely achievable by the hardware system in tissues, informing the cuttoff resolution.
In our specific case, the minimal motions are the minimal insertion and minimal rotation the needle tip can precisely perform, determined by a lower-level needle controller.

\section{The \ros algorithm}
\label{sec:method}

We describe our \ros algorithm for resolution-optimal motion planning. We also highlight the differences between our new algorithm and our previous method \rcs~\cite{Fu2021_RSS}. After presenting \ros, we discuss an additional procedure to further improve its performance.

\subsection{Algorithm description}
The core idea of \ros (and \rcs) is to build a search tree with predefined motions of multiple resolutions.
Specifically, \ros constructs a search tree $\T = (\V, \E)$ embedded in the configuration space, where each node $v \in \V$ is associated with a configuration $\mathbf{x}_v \in \X$ and each edge $e = (u, v) \in \E$ represents a transition from $\mathbf{x}_u$ to $\mathbf{x}_v$.
The search tree is rooted at the given start configuration $\mathbf{x}_{\rm start}$ and explores valid motion plans when expanded in the configuration space. 
See the pseudocode of \ros in Alg.~\ref{alg:shared}.

\begin{algorithm}[t!]
\caption{\ros\\
         \textbf{Input:} $\W_{\rm obs}, 
                          \mathbf{x}_{\rm start}, 
                          p_{\rm goal}, 
                          \tau,
                          \kappa_{\max},
                          \ell_{\max}, 
                          \delta\ell_{\max},
                          R_{\min}$
        }
    \begin{algorithmic}[1]
        \State{$\Theta \leftarrow \{0, \frac{\pi}{2}, \pi, \frac{3\pi}{2}\}, K \leftarrow \{0, \kappa_{\max}\}$}
        \label{line:init_coarse}
        \State{root $\leftarrow (\mathbf{x}_{\rm start}, 0, 0)$}
        \Comment{{\footnotesize The root has rank 0 and cost 0}}
        \State{OPEN $\leftarrow \{$root$\}$, CLOSED $\leftarrow \emptyset$, bestPlan $\leftarrow$ NULL}
        \label{line:init_sets}
\vspace{1mm}
        \While{not OPEN.empty()}
        \label{line:termination}
            \State{$v \leftarrow$ OPEN.extract()}
            \label{line:extract}
\vspace{1mm}
            \If{Valid($v, \W_{\rm obs}, p_{\rm goal}, \ell_{\max}$)}
            \label{line:valid}
                \If{\textbf{not} CLOSED.existDuplicate($v$)}
                \label{line:similar-node}
                    \If{GoalReached($v, p_{\rm goal}, \tau$)}
                    \label{line:goal-reached}
                        \State{bestPlan.update($v$)}
                        \label{line:update-best}
                    \EndIf
\vspace{1mm} 
                    \For{$\M \in$ Primitives($K, \delta\ell_{\max}, \Theta$)}
                    \label{line:expand}
                        \State{OPEN.insert($v \oplus \M$)}
                        \label{line:expand_1}
                    \EndFor
                    \State{CLOSED.insert($v$)}
                    \label{line:add_to_closed}
                \EndIf
            \EndIf
\vspace{1mm} 
            \If{$v$ != root}
            \label{line:check-root}
                \For{$\M \in$ RefinedPrimitives($\M_v$)}
                \label{line:refine-primitive}
                    \If{ValidResolution($\M, R_{\min}$)}
                    \label{line:cutoff}
                        \State{OPEN.insert($v.{\rm parent} \oplus \M$)}
                    \label{line:finer_nodes}
                    \EndIf
                \EndFor
            \EndIf
        \EndWhile
\vspace{1mm} 
        \State{\textbf{return} bestPlan}
        \label{line:final-return}
    \end{algorithmic}
\label{alg:shared}
\end{algorithm}

To generate new nodes, \ros expands existing nodes with predefined (kinematically) feasible motion primitives~\cite{Frazzoli2002_JGCD}.
In \ros, a motion primitive defines a local circular trajectory with some constant curvature
$\kappa \leq \kappa_{\max}$,
some length
$\delta\ell > 0$,
and some curving direction 
$\delta\theta \in [0, 2\pi)$.
That is, a motion primitive is defined as a tuple
$\M = (\kappa, \delta\ell, \delta\theta)$.
We denote the operation of applying a motion primitive~$\M$ to configuration $\mathbf{x}_v$ as
$\mathbf{x}_u = \mathbf{x}_v \oplus \M$,
where $\mathbf{x}_u$ is the resultant configuration.
\ros uses a fixed set of curvatures
$\{0, \kappa_{\max}\}$
and defines the resolution of a motion primitive as a function of 
$\delta\ell$ and $\delta\theta$, since any curvature
$\kappa \in [0, \kappa_{\max}]$ can be well approximated by interleaving curvature $0$ and $\kappa_{\max}$~\cite{Fu2021_RSS}.
Generally speaking, the finer a resolution is, the finer the intervals
$[0, \delta\ell_{\max}]$ and $[0, \delta\theta_{\max}]$ are discretized.
We mention that the coarsest resolution is set with a user-defined $\delta\ell_{\max}$ and $\delta\theta_{\max} = \pi/2$ (4 initial steering orientations as shown in line~\ref{line:init_coarse}).

In each iteration of \ros, an expansion of existing nodes is performed in an \astar-like fashion. In particular, nodes are iteratively extracted from the OPEN list (line~\ref{line:extract}), wherein nodes are ordered according to their \emph{rank} and a secondary metric
$f(\cdot)$.
We define the rank of a node as a function of the node's depth in the tree and the resolution of the motion primitives leading to the node. 
The deeper a node is in the tree and the finer resolution the motion primitives are, the higher rank a node has (see formal definition of rank in~\cite{Fu2021_RSS}).
The secondary metric $f(v) = \C(v) + h(v)$ has $\C(v)$ denoting the cost of the trajectory from the root of $\T$ to $v$ with respect to $\C$ and $h(v)$ being a heuristic function estimating the cost of the trajectory from the node~$v$ to the goal point.
For example, in the case where $\C$ is trajectory length, we have $h(v)$ be the length of the Dubins curve~\cite{LaValle2006_BOOK} on the plane spanned by $\mathbf{x}_v$ and $p_{\rm goal}$.
Unlike in \rcs, where nodes with lower rank are always extracted first, \ros relaxes this ordering by introducing a \textit{look-ahead} parameter denoted as $n_{\rm la} \in \mathbb{N}$ (similar to the idea in~\cite{Lindemann2006_ICRA} and~\cite{MSS18}).
At any time during the search, we denote the minimum rank of nodes in the OPEN list as $r_{\rm open}$.
Then we order all nodes with rank $r \leq r_{\rm open} + n_{\rm la}$ according to a secondary metric $f(\cdot)$. This is done to prioritize searching nodes from a coarser resolution, which speeds up finding an initial motion plan (this is similar in nature to using a focal list~\cite{PK82} in \astar-like algorithms).

Given an extracted node $v$, we first check if it is valid (line~\ref{line:valid}) using the conditions described for \rcs, which ensure that (i) the insertion length would not exceed $\ell_{\max}$, (ii) the goal region is still reachable after getting to $v$, 
(iii) the trajectory from the root to $v$ is not identical to another node that only needs coarser motion primitives to get to, and (iv) that the edge leading to $v$ is collision free~\cite{Fu2021_RSS}. In addition, \ros checks that the cost $\C(v)$ is smaller than the cost of the best plan reaching the goal region found so far. 
If the heuristic function~$h(\cdot)$ is admissible, we use $f(v)$ instead of~$\C(v)$ in the last condition, as $f$ provides a better estimate of the node cost and hence allows to prune more vertices.

To further boost efficiency, the algorithm avoids expanding nodes that are highly similar to existing nodes in terms of the induced configuration and cost by performing duplicate detection (line~\ref{line:similar-node}). 
A node $v$ is determined as a duplicate if there exists a node $u$ in the CLOSED list that satisfies (i')~$\rho(\mathbf{x}_u, \mathbf{x}_v) < d_{\rm sim}$ and (ii')~$\C(u) \leq \C(v)$,
where $\rho$ is a distance function defined on $\X$ and $d_{\rm sim}$ is a similarity parameter.
We use $\rho(\mathbf{x}_u, \mathbf{x}_v) = \|p_u - p_v\|_2 + \alpha\cdot{\rm dist}_\sphericalangle(q_u, q_v)$,
where $\alpha > 0$ is a weighting parameter and ${\rm dist}_\sphericalangle()$ is the angular distance between two orientations.
Sec.~\ref{sec:theory} specifies the value of $d_{\rm sim}$.
Condition (i') is shared between \rcs and \ros, while 
condition (ii'') is important for keeping \ros resolution optimal as it prevents nodes with lower cost from being pruned away by nodes with higher cost.

\ros uses a set of motion primitives of different resolutions, but instead of applying all motion primitives together, only the coarsest motion primitives are used when a node is initially expanded (line~\ref{line:expand}).
The resolution of a motion primitive (with respect to $\delta\ell$ and $\delta\theta$) is refined when a node with a coarser motion primitive is processed (line~\ref{line:refine-primitive}).
More specifically, a finer motion primitive is obtained by changing $\delta\ell$ or $\delta\theta$ by a small value that corresponds to a finer resolution (see~\cite{Fu2021_RSS} for the exact definitions).
Note that resolution refinement is done even if a node is invalid since finer resolutions might be valid.

Resolution refinement will be cut off when reaching a finer resolution, with respect to $\delta\ell$ or $\delta\theta$, than the predefined cutoff resolution $R_{\min}$ (line~\ref{line:cutoff}).
Here the cutoff resolution is determined by the minimal motions of the needle tip that are precisely achievable with the hardware system.

The algorithm terminates when the OPEN list is exhausted (line~\ref{line:termination}), and the best plan is returned (if any is found).
\ros is guaranteed to terminate in finite time due to the cutoff resolution.

\subsection{Domain-specific optimization}
\label{sec:domain-opt}

\begin{figure}
    \centering
    \includegraphics[width=\linewidth]{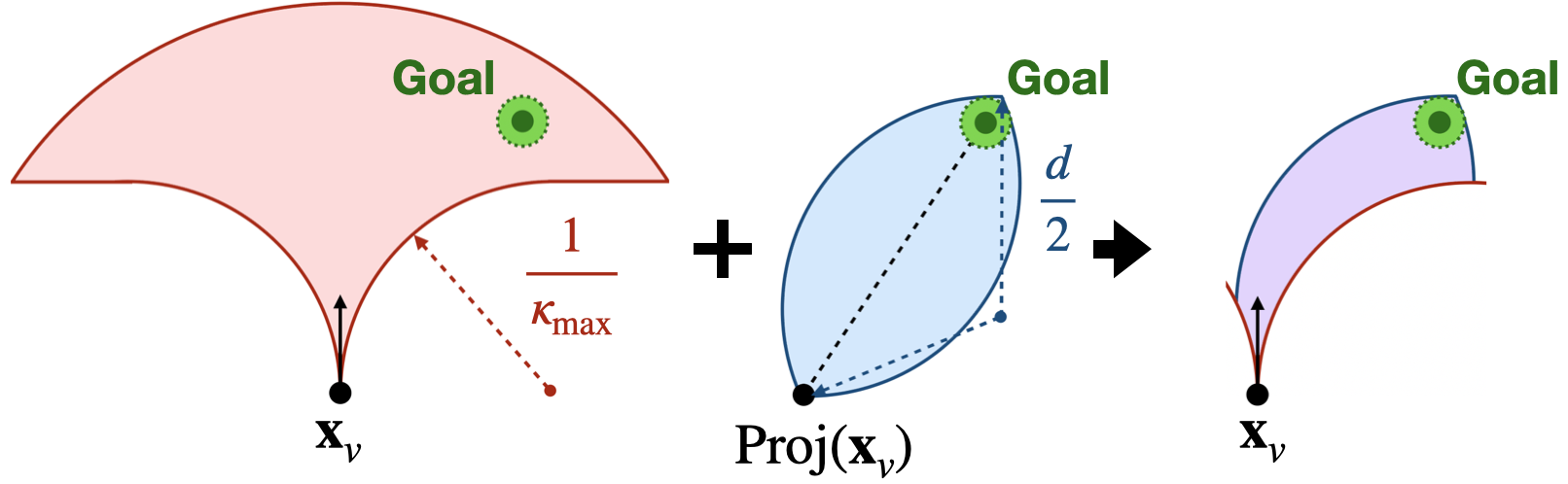}
    \caption{
    A 2D illustration of the approximated reachable workspace.
    The kinematically forward-reachable workspace is shaded in red (a 3D version can be obtained by rotating the region around the tangent vector at $\mathbf{x}_v$, which results in a trumpet shape).
    The feasible workspace is shaded in blue.
    The diameter of the circular arcs is
    $d = {\max}(2/\kappa_{\max}, \tau + \|{\rm Proj}(\mathbf{x}_v) - p_{\rm goal}\|_2)$.
    The final approximated reachable workspace is shaded in purple.
    }
    \label{fig:reachable-space}
\end{figure}

We describe an additional procedure to further improve \ros's performance.
We incorporate the concept of inevitable collisions~\cite{LaValle2006_BOOK} to eliminate potential nodes that would lead to collisions as they are expanded. 
In particular, for a given vertex~$v$ and the goal point, region growing is performed from $\mathbf{x}_v$ within an approximated reachable workspace, considering the existence of obstacles.
This region is defined as the intersection of the kinematically forward-reachable workspace and the olive-shaped feasible workspace defined by $\mathbf{x}_v$, $p_{\rm goal}$, and tolerance~$\tau$ (see Fig.~\ref{fig:reachable-space}). We mention that due to
(i)~maximum curvature constraint,
(ii)~maximum turning angle constraint (the needle would shear or buckle when turning over $\pi/2$),
and
(iii)~maximum insertion length constraint, the kinematically forward-reachable workspace for a given needle configuration is a trumpet-shaped volume 
(see Fig.~\ref{fig:reachable-space} left).
In the case that the goal is not reached by the grown region, $v$ is discarded.
For additional optimizations applicable to \ros, see~\ref{sec:appendix_optimization}.

\section{Resolution Optimality of \ros}
\label{sec:theory}

We study the theoretical properties of \ros and provide a proof for the algorithm being \emph{resolution optimal}.
Informally, resolution optimality implies that \ros is guaranteed to find a plan whose cost is as close as desired to the cost of the globally optimal qualified motion plan $\sigma^*$, assuming that the cutoff resolution $R_{\min} = \{\delta\ell_{\min}, \delta\theta_{\min}\}$ is fine enough.
Thm.~\ref{thm:resolution_optimal} given below, states our main theoretical contribution relating to the resolution optimality of \ros.

Before stating Thm.~\ref{thm:resolution_optimal}, we introduce the notions of a \emph{well-behaved cost}, as well as \emph{robust} and \emph{decomposable} trajectories, which will be used to state the necessary conditions on~$\C$ and~$\sigma^*$ required to prove our result.
The crux of the problem is that it may not be possible to approximate the optimal plan $\sigma^*$ using motion plans with a finite number of transitions without additional (realistic) constraints on the cost function $\C$ and on~$\sigma^*$.

To this end, we start to define the notion of a Well-behaved cost that states that close-by configurations have similar costs and that there are bounds on the values that the cost can attain.

\begin{dft}[Well-behaved cost]
    A configuration-based cost function~$c$ is well-behaved if
        (i) it is Lipschitz continuous, i.e.,
              $\forall \mathbf{x}_1, \mathbf{x}_2 \in \X_{\rm free}, \vert c(\mathbf{x}_1) - c(\mathbf{x}_2)\vert \leq L_c \cdot \rho(\mathbf{x}_1, \mathbf{x}_2)$ for $L_c \in \mathbb{R}$,
        and (ii) $\forall \mathbf{x} \in \X_{\rm free}, c(\mathbf{x}) \in [c_{\min}, c_{\max}]$, where $c_{\min}, c_{\max} \in \mathbb{R}$ and $c_{\min} > 0$.
    In such a case, we also say that the trajectory-based cost function $\C(\sigma) = \int_{0}^{\ell} {c\big(\sigma(s)\big)\mathrm{d}s}$ is well-behaved.
\end{dft}

We assume $c_{\max}$ is not infinitely large since such configurations can be removed from $\X_{\rm free}$.
We also require a well-behaved cost to satisfy $c_{\min} > 0$ since in the case of needle steering, there is always a cost associated with puncturing tissue.

Next, we provide two definitions that are used to characterize motion plans that \ros can approximate. 
The first definition (borrowed from~\cite{Fu2021_RSS}) is concerned with trajectories that are induced by a finite set of motion primitives (not necessarily the ones used by \ros). 
The second definition is concerned with so-called robust trajectories that admit some clearance from the obstacles.
A motion plan is then considered \textit{qualified} if it satisfies both definitions.

\begin{dft}[Decomposable trajectory]
A trajectory $\sigma:[0,l]\rightarrow \X$ is decomposable if it can be decomposed into a finite set of motion primitives. Namely, there exist primitives $M_\sigma = \{\M_1,\ldots, \M_n\}$ such that $\sigma = \sigma(0) \oplus M_\sigma $,
where $\mathbf{x} \oplus M$ denotes the resultant trajectory obtained by sequentially applying elements in $M$ to $\mathbf{x}$.
\end{dft}

\begin{dft}[Robust trajectory]
A trajectory $\sigma:[0,l]\rightarrow \X$ is $\delta$-robust, for some $\delta>0$,  if 
(i)~it has $\delta$ clearance from obstacles, i.e., 
$\min_{s\in [0,l], \mathbf{x} \in \X_{\rm obs}}\rho(\sigma(s), \mathbf{x}) > \delta$,  
and if 
(ii)~it's endpoint is within a distance of $\tau - \delta$ to the goal. Namely, $\|{\rm Proj}(\sigma(l)) - p_{\rm goal}\|_2 < \tau - \delta$.
Here, $\X_{\rm obs} = {\rm cl}(\X \setminus \X_{\rm free})$ and ${\rm cl(\cdot)}$ is the closure of a set.
Note that, we implicitly assume here that $\tau  > \delta$.
\end{dft}

We are ready to state our main theoretical result concerning the resolution optimality of \ros. 

\begin{thm}[Resolution optimality]
\label{thm:resolution_optimal}
    Let $\Delta = (\X, \W_{\rm obs}, \mathbf{x}_{\rm start}, p_{\rm goal}, \tau, \ell_{\max}, \kappa_{\max}, \C)$ be an optimal steerable needle motion-planning problem, $\varepsilon\in (0,\infty)$ be an approximation factor,
    and $\sigma^*$ be a trajectory. 
    Also, suppose that the following conditions are satisfied:
    \begin{enumerate}[label=(\textbf{C\arabic*})]
        \item 
        \label{C1}
        The steerable-needle system is Lipschitz continuous and characterized with $L_s$ (see~\ref{sec:appendix} for formal definition);
        \item
        \label{C2}
        The cost function~$\C$ is well-behaved and characterized with $L_c, c_{\min}, c_{\max}$.
        Denote $k = \frac{L_c + c_{\max}}{c_{\min}}$.
        \item 
        \label{C5}
        The optimal trajectory $\sigma^*$ is decomposable and $\delta$-robust with $\delta = {\min}\{\frac{\varepsilon}{k}, \frac{\tau}{2}\}$.
        \item
        \label{C6}
        The radius~$d_{\rm sim}$ used to reject similar nodes satisfies
    \end{enumerate}
    {\footnotesize
    $$
        d_{\rm sim} < {\min}
        \bigg\{
        \frac{2}{\kappa_{\max}}\sin{\frac{\kappa_{\max}\delta\ell_{\min}}{2}},
         \frac{\delta(L_s - 1)}{2(L_s^H - 1)}
        \bigg\}, \textup{where }
        H = \left\lceil \frac{\ell_{\max}}{\delta\ell_{\min}}\right\rceil.
    $$
    }
    Then \ros  is resolution optimal, i.e., for a fine-enough cutoff resolution~$R_{\min} = \{\delta\ell_{\min}, \delta\theta_{\min}\}$, \ros
    will find a motion plan that satisfies 
    $\C(\sigma) \leq (1 + \varepsilon)\cdot\C(\sigma^*)$.
\end{thm}

\subsection{Proof sketch for Thm.~\ref{thm:resolution_optimal}}

We provide a sketch of the proof for Thm.~\ref{thm:resolution_optimal} and defer the full proof to~\ref{sec:appendix} .
The proof sketch consists of two main steps:  We first show that the plan $\sigma^*$
can be approximated by another plan $\sigma^*_\beta$ that is composed solely of the motion primitives used by \ros. Then, we show that even though \ros might not be able to exactly find $\sigma^*_\beta$ due to pruning, it will be able to recover another plan $\tilde{\sigma}^*_\beta$ whose cost is similar to that of $\sigma^*_\beta$ (and $\sigma^*$).

We assume that the conditions in Thm.~\ref{thm:resolution_optimal} are met for a plan~$\sigma^*$ and the approximation factor $\varepsilon>0$. 
As a first step, we show that there exists a resolution $R_{\min}$ and a trajectory~$\sigma^*_\beta$ that approximate $\sigma^*$ such that $\sigma^*_\beta$ can be constructed solely from the motion primitives of \ros. 
In particular, $\sigma^*_\beta$ is a \emph{piece-wise strict $\beta$-approximation} of~$\sigma^*$ which is defined as follows: the two trajectories $\sigma^*$ and $\sigma^*_\beta$ can be partitioned into a sequence of trajectories $\sigma_1,\ldots, \sigma_n$ and $\sigma'_1,\ldots, \sigma'_n$, respectively, for some positive and finite integer $n$, such that for any $1\leq i\leq n$ we have that 
(i)~the Hausdorff distance between $\sigma_i$ and $\sigma'_i$ is below some maximal threshold~$\beta>0$, 
and that 
(ii)~the cost of the two trajectories satisfies $\C(\sigma'_i)\leq (1+\beta)\cdot\C(\sigma_i)$ (see formal definition in~\ref{sec:appendix}).
This step extends and generalizes our previous proof~\cite[Lemma ~2]{Fu2021_RSS} that showed the existence of such an approximation, albeit only for the length cost, and relies on assumptions \textbf{C1} and \textbf{C2}, which characterize the system and cost function, and on  assumption~\textbf{C3}, which states that $\sigma^*$ is decomposable.

The value $\beta$ is chosen to guarantee that $\C(\sigma^*_\beta)\leq (1+\varepsilon)\cdot\C(\sigma^*)$. Moreover, the choice of $\beta$, and the fact that $\sigma^*_\beta$ is a $\beta$-approximation of $\sigma^*$ also ensure that $\sigma^*_\beta$ is collision free and satisfies the goal tolerance according to assumption \textbf{C3}. Note that the above properties would still hold even if we replace the constant $\beta$ with a slightly larger value $\beta'>\beta$. However, we use the more conservative value $\beta>0$ to compensate for the fact that \ros prunes the tree. In particular, due to pruning, \ros might eliminate some of the vertices induced by the trajectory $\sigma^*_\beta$, and thus, we cannot guarantee that $\sigma^*_\beta$ would be returned as a solution. However, we can show that even in the presence of pruning \ros will compute a valid plan $\tilde{\sigma}^*_\beta$ that tightly bounds the cost of $\sigma^*_\beta$ (and thus tightly bounds the cost of $\sigma^*$).

We now elaborate on this.
Denote the sequence of motion primitives that define $\sigma_{\beta}^*$ as 
$M_{\sigma_{\beta}^*} = \{\M_1,\dots,\M_n\}$.
When  the motions in
$M_{\sigma_{\beta}^*}$ are sequentially applied to 
$\mathbf{x}_{\rm start}$, 
we obtain a sequence of configurations
$\{\mathbf{x}_0,\mathbf{x}_1,\dots,\mathbf{x}_n\}$, 
where $\mathbf{x}_0 = \mathbf{x}_{\rm start}, \mathbf{x}_i = \mathbf{x}_{i-1} \oplus \M_i, i \in [1, n]$, some of which may be pruned.
By carefully bounding $d_{\rm sim}$ (according to~\textbf{C4}), we guarantee that in the worst case the trajectory obtained by applying the same sequence of motion primitives to pruned nodes stays within a collision-free ``tunnel'' around~$\sigma_\beta^*$, ensuring that it remains valid.
The choice of $d_{\rm sim}$ also takes care of goal tolerance, guaranteeing that the trajectory still satisfies $\tau$ goal tolerance.
Finally, because (i)~pruning is allowed only when there exists a node with equal or smaller cost, and 
(ii)~the subtrajectories obtained from applying the same motion primitive to close-by configurations are strict approximations, the plan $\tilde{\sigma}^*_\beta$ tightly bounds the cost of $\sigma^*_\beta$.

\section{Results}
\label{sec:results}

\begin{figure*}
    \centering
    \includegraphics[width=\linewidth]{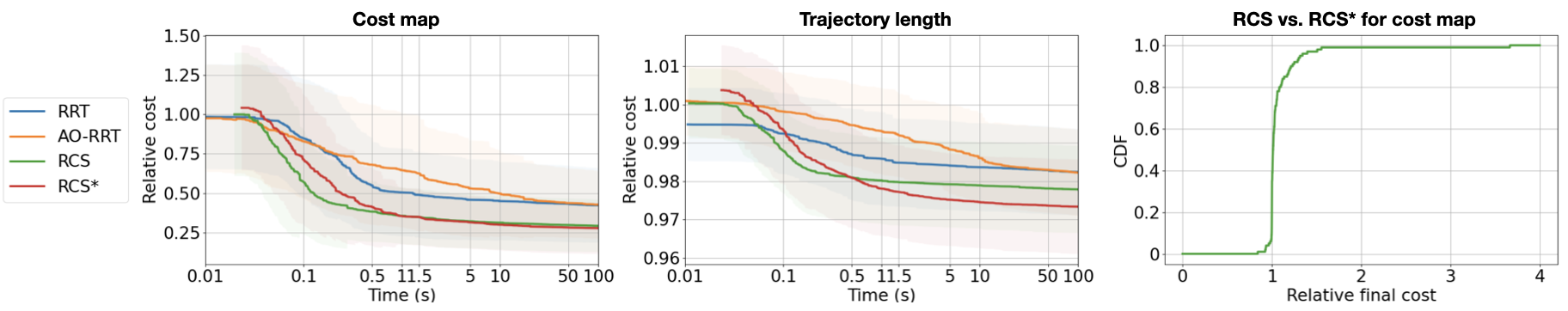}
    \caption{
    \textbf{Left and middle:}
    Performance comparison of the different motion planning methods with two different cost functions.
    The time axes use a logarithmic scale. The shaded regions show standard deviations.
    \textbf{Right:}
    A detailed comparison of \rcs and \ros for the cost-map setting, shown as a cumulative distribution function (CDF).
    }
    \label{fig:results}
\end{figure*}

We focus on the medical procedure of lung biopsy for evaluation.
Lung and bronchus cancer has the highest death rate among all types of cancer, killing over 130,000 Americans each year~\cite{ACS2022}.
Lung biopsy enables definitive diagnosis of suspicious lung nodules at an early stage, which is important to increase the survival rate.
One potential approach to safely and accurately access lung nodules for biopsy and localized treatment is deploying a steerable needle trans-orally through a bronchoscope to avoid transthoracic access which could cause severe side effects~\cite{Kuntz2016_Hamlyn,Swaney2017_JMRR, Hoelscher2021_RAL, Fried2021_ICRA}. 
In this approach, a physician deploys a bronchoscope through a patient's bronchial tubes and then the steerable needle is deployed through the bronchoscope, exits out of the bronchial tube, and steers in the lung parenchyma to reach the nodule.
Our motion planner focuses on the final stage of automatically steering the needle through the lung parenchyma to the nodule, while avoiding anatomical obstacles including large blood vessels, bronchial tubes, and the lung boundary.

We used the method developed in~\cite{Fu2018_IROS} to reconstruct the anatomical workspace from a chest CT scan with the above-mentioned obstacles.
We collected $100$ test cases. 
For each test case, we randomly sampled a start pose close to the bronchial tubes and a target in the lung parenchyma.
We finished collecting when the number of test cases reached $100$ after rejecting (i) impossible scenarios where the start pose has inevitable collision and (ii) trivial scenarios where the start pose can be connected directly to the goal point with a collision-free arc.
The simulated needle has 
$\kappa_{\max} = (50 {\rm mm})^{-1}, \ell_{\max} = 100 {\rm mm}$,
and a diameter of $2{\rm mm}$.
We set the goal tolerance $\tau = 1.0 {\rm mm}$.

We compared \ros in simulation with several planners:
\begin{enumerate}[label=(\roman*)]
    \item \textbf{\rrt:} The \rrt-based needle planner~\cite{Patil2014_TRO, Kuntz2015_IROS} with $5\%$ goal biasing and $100\%$ goal connecting ratio. 
    \item \textbf{\aorrt:} \aorrt~\cite{Hauser2016_TRO, Kleinbort2020_ICRA} adapted for steerable needles, with maximum rotation control $2\pi$ and a maximum insertion control $20 {\rm mm}$. We follow the guidelines in \cite{Kleinbort2020_ICRA} for cost sampling and distance weighting between the configuration space and cost space.
    For a fair comparison, we use the same goal connecting as \rrt. 
    \item \textbf{\aft:} The \aft-based needle planner~\cite{Liu2016_RAL, Pinzi2019_IJCARS}, with setup following \cite{Pinzi2019_IJCARS}.
    \aft internally uses a hybrid cost function; we use 
    $\C_{\rm hybrid}(\sigma) = \omega \cdot \C(\sigma) + \|\sigma(\ell_\sigma) - p_{\rm goal}\|_2/\tau$, where $\omega$ is a weighting parameter depending on the scale of $\C$. Note that $\C_{\rm hybrid}$ is only used internally in \aft while $\C$ is always used for performance comparison across different planners.
    \item \textbf{\rcs:} The \rcs needle planner~\cite{Fu2021_RSS} with cutoff resolution $R_{\min} = \{0.125 {\rm mm}, 0.157 {\rm rad}\}$, $\delta\ell_{\max} = 20 {\rm mm}$.
\end{enumerate}
\ros used the same basic setup as \rcs and $n_{\rm la} = 3$.
Additionally, for all cost functions, we set $\varepsilon = 0.1$.
Since we used a CPU implementation of \aft, we did not compare its performance over time, and only reported the final cost after three iterations, as suggested in~\cite{Liu2016_RAL}.
For each method, except for \aft, the timeout was set to $100$ seconds.
All methods except for \aft achieved $100\%$ success rate when timed out while \aft achieved $87\%$ success rate.
When reporting the relative cost of \aft, we only consider the test cases that are successfully solved by \aft.
As \rrt and \rcs are designed to find a plan instead of optimizing a plan, we modified them to keep running after the first plan is found, and always return the best plan found when timed out.
All experiments were run on a dual 2.1GHz 16-core Intel Xeon Silver 4216 CPU and 100GB of RAM.
All parallelizations were implemented with Motion Planning Templates (MPT)~\cite{Ichnowski2019_ICRA}.

We used two well-behaved cost functions for which \ros is resolution optimal:
(i)~trajectory length, and (ii)~a cost function informed by a cost map derived from medical images~\cite{Fu2018_IROS}, where each voxel in the 3D cost map is associated with a cost value that represents tissue damage.
We used trilinear interpolation to smooth out the voxelized cost map and forced $c_{\min} = 0.01$.

Fig.~\ref{fig:results} shows results for the different planners for the cost map and trajectory length cost functions.
For both cost functions, the cost of a plan may vary significantly between test cases.
For example, trajectory length is affected by how far away the target lies relative to the start pose and cost map values are much higher when the needle is steering in a vessel-cluttered region.
To account for the large variation between test cases, we first computed \emph{relative cost} within one test case.
Specifically, we took the first-found plan (no matter which method returns it as long as it is returned the fastest), and computed relative cost to this plan for all other plans found for the same test case.
We did this for all the test cases, and then averaged over the 100 test cases (see the first two plots in Fig.~\ref{fig:results}).
Such relative cost decreased as the result plans are gradually optimized.
\aft, which is omitted from those plots, achieved a final cost of $0.574$ and $0.991$ for cost map and trajectory length, respectively, after three iterations.

For both cost functions, \ros achieved the best final costs. 
For the cost map, which is more clinically relevant, \ros outperformed other methods after $1.5$ seconds with a final cost of $0.277$, which is $50\%$ lower than \aft ($0.574$), $35\%$ lower than \rrt ($0.423$) and \aorrt ($0.427$), and $7\%$ lower than \rcs ($0.299$).
This indicates that the final plan produced by \ros successfully avoids more small blood vessels than the other methods. It is worth mentioning that unlike trajectory length, small perturbations may lead to very different costs when we use the cost map. 
As for trajectory length, all methods generated trajectories with roughly similar lengths, although \ros computed shorter trajectories than all other methods. 

In Fig.~\ref{fig:results}, we report the cumulative distribution function across the 100 test cases of the relative final cost of \rcs with respect to \ros for the cost-map setting. Although \rcs achieved comparable costs with \ros in most cases, for about $20\%$ of the test cases \rcs achieved a motion plan that is at least $10\%$ more costly than \ros. 
In the extreme case, the final cost of \rcs was $3.67$ times that of \ros.

Finally, we mention that the number of nodes \ros explored is only $53\%$ (for cost map) and $30\%$ (for trajectory length) of that \rcs explored.
This indicates that although \ros spent more time choosing the right node to explore, it explores much fewer nodes than \rcs to get a plan with even higher quality.
The experimental results demonstrate that \ros is faster than competing methods and that the theoretical guarantees of \ros have a practical impact on the quality of result plans, which is valuable for computing motion plans that minimize patient trauma.

\section{Conclusion}
\label{sec:conclusion}

In this paper, we introduce the first resolution-optimal motion planner for steerable needles.
In particular, our method returns in finite time a motion plan whose cost can be as close as desired to the globally optimal qualified motion plan, assuming the given resolution is fine enough.
We also provide a proof sketch to show the resolution optimality of our method with a careful discussion of assumptions and required conditions.
We evaluate our proposed planner with simulation experiments and show it can efficiently compute high-quality motion plans considering clinically relevant cost functions.
In the future, we plan to investigate speedup techniques for vertex validation and pruning to reduce vertex expansion time.
In addition, we plan to develop explicit expressions for the resolution $R_{\min}$ necessary to achieve the desired level of approximation quality for a given problem instance.
We will also experimentally evaluate the planner with steerable needles in ex-vivo tissues.


\bibliographystyle{IEEEtran}
\bibliography{IEEEabrv, refabrv, references}

\clearpage

\setcounter{section}{0}
\renewcommand{\thesection}{Appendix \Alph{section}}

\setcounter{subsection}{0}
\renewcommand{\thesubsection}{\Alph{subsection}}

\setcounter{figure}{0}
\renewcommand{\thefigure}{\Alph{figure}}

\section{Domain-specific optimizations}
\label{sec:appendix_optimization}

\begin{figure}
    \centering
    \includegraphics[width=\linewidth]{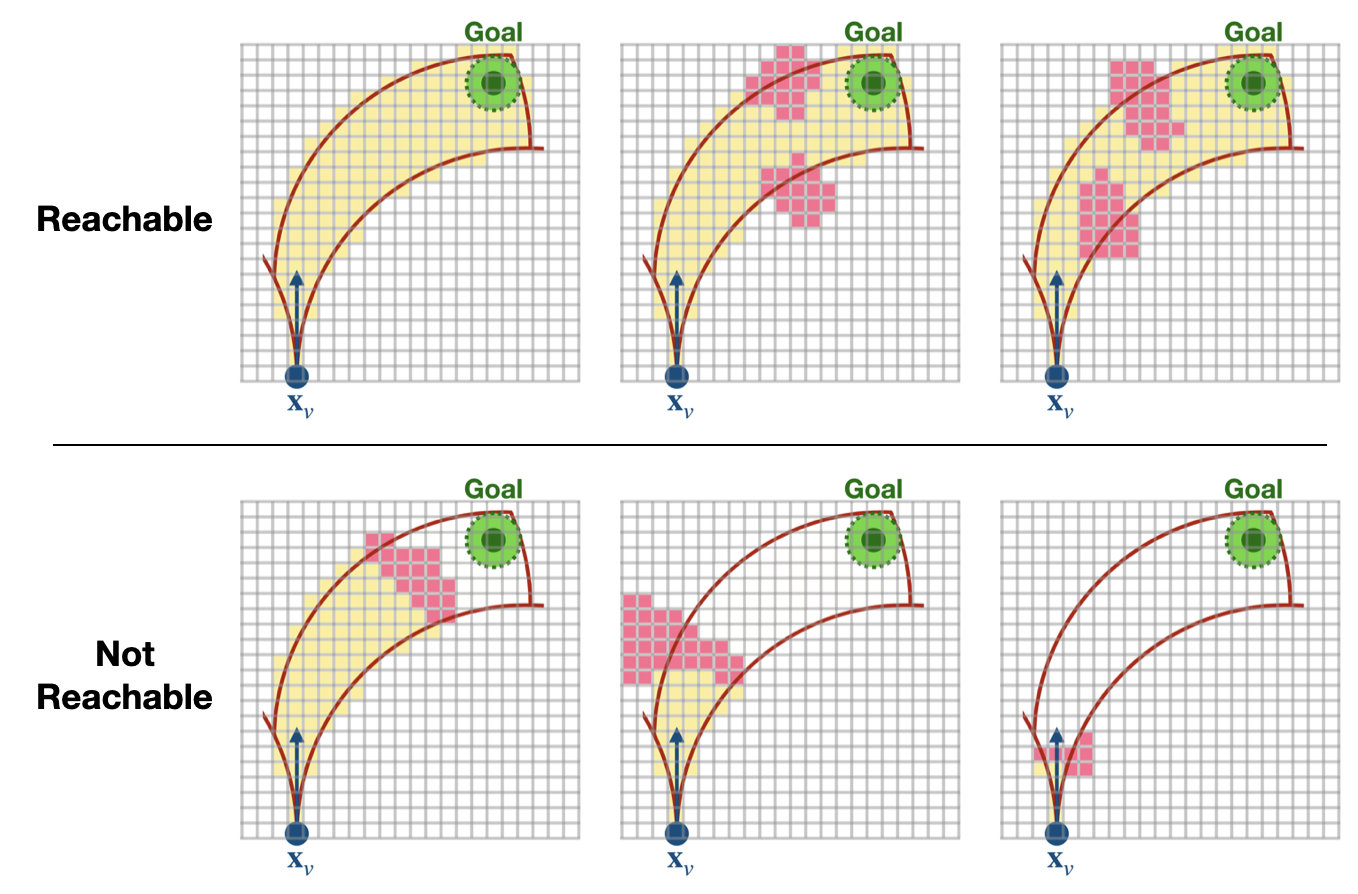}
    \caption{
    2D illustration of example cases for inevitable collision check.
    The connected region is shaded yellow, obstacle voxels are shaded pink.
    This is an over estimation so even if the goal is determined as reachable in the check, it is not guaranteed that a valid motion plan to the goal exists.
    }
    \label{fig:examples}
\end{figure}

We describe additional domain-specific optimizations to further improve \ros's performance.

First we elaborate on the a method accounting for inevitable collisions~\cite{LaValle2006_BOOK} to eliminate potential nodes that are bound to lead to collisions as they are expanded, that we briefly described in Sec.~\ref{sec:domain-opt}.
In particular, for a given vertex~$v$ and the goal point, a ``region-growing'' process is performed from $\mathbf{x}_v$ within an approximation of the reachable workspace, while considering the existence of obstacles.
This region is defined as the intersection of the kinematically forward-reachable workspace and the olive-shaped feasible workspace defined by $\mathbf{x}_v$, $p_{\rm goal}$, and tolerance~$\tau$ (see Fig.~\ref{fig:reachable-space}).

We mention that due to 
(i)~maximum curvature constraint,
(ii)~maximum turning angle constraint (the needle would shear or buckle when turning over $\pi/2$),
and
(iii)~maximum insertion length constraint, the kinematically forward-reachable workspace for a given needle configuration is a trumpet-shaped volume rooted at the current needle position (see Fig.~\ref{fig:reachable-space}  left).

Additionally, a position in the workspace is potentially feasible only when there exists some orientation with which the goal region is forward reachable while the start point is backward reachable, considering the maximum curvature constraint.
This defines the olive-shaped feasible workspace (see Fig.~\ref{fig:reachable-space} middle) since for any position outside the region, there is no orientation that is valid.

In the case that the goal is not reached by the grown region, $v$ is discarded.
Several examples are provided in Fig.~\ref{fig:examples}.
Inevitable collision check allows us to efficiently identify and discard invalid branches without refining the resolution.

In addition to the above-mentioned optimization, all domain-specific optimizations introduced for \rcs~\cite{Fu2021_RSS} are also applicable to \ros.
More specifically, the optimizations are as follows.
(i)~Early pruning by testing for goal reachability. I.e., we check if the goal region overlaps with the kinematically forward-reachable workspace for a configuration.
(ii)~Direct goal connection. I.e., we try connecting a configuration to the goal point with a kinematically-feasible trajectory, and a plan is found if the trajectory is collision free. In \rcs, a circular arc with constant curvature is used while in \ros, we use the trajectory with the  shortest length, which is the Dubins curve~\cite{LaValle2006_BOOK}.
(iii)~Equivalent node pruning. I.e., we avoid duplicated nodes caused by refining $\delta\ell$ and $\delta\theta$ in a different order.
(iv)~Parallelism. I.e., we process multiple nodes in parallel to improve the overall runtime.
Refer to~\cite{Fu2021_RSS} for more details.

\section{Full Proof for the Resolution Optimality of \ros}
\label{sec:appendix}
We provide  a full proof of Theorem~\ref{thm:resolution_optimal}, which states that \ros is resolution optimal. Our proof below follows the proof sketch we provided in Section~\ref{sec:theory}. 
We start this section by providing some definitions that were omitted in the proof sketch. We then establish a basic property stating that decomposable trajectories can be approximated by a finite sequence of motion primitives. We then exploit this property in the full proof of Thm.~\ref{thm:resolution_optimal}.

\subsection{Preliminaries}
\label{sec:preliminaries}

We provide a formal definition of the notion of piece-wise strict $\beta$-approximation, which we briefly described in the proof sketch. We first define this notion for a single local ``piece'' in the following definition, and then extend it for trajectories consisting of several pieces. 

\begin{dft}[Local strict approximation]
\label{dft:strict-approx}
    For two trajectories $\sigma : [0, l] \rightarrow \X$ and $\sigma' : [0, l'] \rightarrow \X$, and a value $\beta>0$,  we say $\sigma'$ is a local strict $\beta$-approximation of $\sigma$ if 
    \begin{enumerate}[label=(\roman*)]
        \item $l' \leq (1 + \beta)\cdot l$,
        \item $\forall s \in [0, {\min}(l, l')), \rho\left(\sigma(s), \sigma'(s)\right) \leq \beta$,
        \item $\forall s \in [{\min}(l, l'), l'], \rho\left(\sigma(l), \sigma'(s)\right) \leq \beta$.
    \end{enumerate}
\end{dft}

Note that this definition is stricter than bounding the two-way Hausdorff distance between two trajectories since it requires bounded trajectory length and bounded distances between corresponding points on the trajectories. 
This is a major difference from our previous definition~\cite{Fu2021_RSS}, where we used non-strict approximations. This will be an important tool for bounding the cost of the trajectory computed by \ros. 
Next, we extend this notion to multiple pieces. 

\begin{dft}[Piece-wise strict approximation]
\label{dft:piecewise-strict-approx}
    For two trajectories $\sigma : [0, l] \rightarrow \X$ and $\sigma' : [0, l'] \rightarrow \X$, and a value $\beta>0$, we say $\sigma'$ is a piece-wise strict $\beta$-approximation of $\sigma$ if 
    there exists two sequences
    $\{s_0, s_1,\dots,s_n\}$ and $\{s_0', s_1',\dots,s_n'\}$
    that
    \begin{enumerate}[label=(\roman*)]
        \item $s_0 = 0, s_0' = 0$,
        \item $s_n = l, s_n' = l'$, and
        \item $\forall i \in [0, n-1]$, the sub-trajectory  $\sigma'(s_i', s_{i+1}')$ is a local strict $\beta$-approximation of $\sigma(s_i, s_{i+1})$.
    \end{enumerate}
\end{dft}

Next, we provide a formal definition of Lipschitz continuity of the steerable-needle system, which is used in \textbf{(C1)} of Thm.~\ref{thm:resolution_optimal}. Following~\cite{Fu2021_RSS}, we define the action space (or motion space)~$\A \in \mathbb{R}^3$ to be the set of all feasible motion primitives.
A motion primitive~$\M = (\kappa, \delta\ell, \delta\theta)$ is feasible if 
$\kappa \in [0, \kappa_{\max}], \delta\ell > 0, \delta\theta \in [0, 2\pi)$. We then define Lipschitz continuity in our primitive-based setting, which is based on the following primitive-based metric.

\begin{dft}[Primitive-based metric $\rho_\A$]
\label{dft:action-distance}
    We define a distance metric on an action space~$\A$ as a distance between two resultant trajectories from
    applying $\M_1$ and $\M_2$ to
    $\mathbf{x}$,
    where $\M_1,\M_2\in \A$ and $\x\in \X$.
    Formally, we have
    \begin{equation*}
    \begin{split}
        &\rho_{\A}(\M_1, \M_2)\\
        &:= {\max}
        \Big\{
        \max_{s \in [0,l_{\min})}
        \rho\big(\sigma_{\M_1}(s), \sigma_{\M_2}(s)\big),\\
        & \max_{s \in [l_{\min},l_{\max}]}
        \rho\big(\sigma_{\M_1}({\min}\{s, l_1\}), \sigma_{\M_2}({\min}\{s, l_2\})\big)
        \Big\},
    \end{split}
    \end{equation*}
    where
    $\sigma_{\M_1}, \sigma_{\M_2}$ are the resultant trajectories,
    $l_1, l_2$ are the trajectory length of $\sigma_{\M_1}$ and $\sigma_{\M_2}$, respectively, 
    $l_{\min} = {\min}\{l_1, l_2\}$, and
    $l_{\max} = {\max}\{l_1, l_2\}$.
\end{dft}

In the above definition note that changing the initial configuration $\mathbf{x}$ does not change the relative position between the two trajectories.
    Thus, without loss of generality, we have 
    $\mathbf{x} = (p, q)$
    where
    $p = (0,0,0)$
    and
    $q = (1,0,0,0)$.

\begin{dft}[Lipschitz continuous]
\label{dft:lipschitz}
    The system is Lipschitz continuous if
    for any $ \mathbf{x}_1, \mathbf{x}_2 \in \X,  \M_1, \M_2 \in \A$, it holds that 
    \begin{equation*} \rho(\mathbf{x}_1 \oplus \M_1, \mathbf{x}_2 \oplus \M_2) \leq L_s\big(\rho(\mathbf{x}_1, \mathbf{x}_2) + \rho_{\A}(\M_1, \M_2)\big),
    \end{equation*}   
    for some constant $L_s > 0$.
\end{dft}

Finally, we introduce the notion of the \textit{finest set} of motion primitives.
\begin{dft}[Finest set of motion primitives]
\label{dft:finest_set}
    Given a resolution $R = \{r_\ell, r_\theta\}$,
    and a set of curvatures~$K$,
    we define the \emph{finest set of motion primitives} as
    \begin{equation*}
        M_{\rm fs}(R,K) = 
        \bigg\{
         (\kappa, r_\ell, n\cdot r_\theta) 
         ~\Big|~
         \kappa \in K,
         n \in \left[ 
                    0, \left \lfloor\frac{2\pi}{r_{\theta}}\right \rfloor   
               \right] \subset \mathbb{Z}
        \bigg\}.
    \end{equation*}
\end{dft}

\subsection{Approximation of decomposable trajectories}
We temporarily set aside the study of \ros's behavior. We prove the following basic result showing that  any decomposable trajectory can be approximated to any desirable degree by a finite sequence of motion primitives. 

\begin{thm}
\label{thm:decomposition}
    Let~$\sigma$ be a decomposable trajectory
    and let~$\beta > 0$ be some real value.
    If the system is Lipschitz continuous, 
    there exists a fine resolution $R(\sigma, \beta) = \{r_{\ell}, r_{\theta}\}$
    and a finite sequence of motion primitives 
    $M_{R(\sigma, \beta)} \subseteq M_{\rm fs}(R(\sigma, \beta), \{0, \kappa_{\max}\})$
    such that  
    $\sigma':=\sigma(0) \oplus M_{R(\sigma, \beta)}$
    is a piece-wise strict $\beta$-approximation of $\sigma$. Moreover, if $c$ is a well-behaved cost (characterized with $L_c, c_{\min}, c_{\max}$),
    then the trajectory cost  satisfies
    $\C(\sigma') \leq (1 + k\cdot\beta) \cdot \C(\sigma)$,
    where $k = \frac{L_c + c_{\max}}{c_{\min}}$.
\end{thm}

We break the proof of this result into the following steps.

\subsubsection{Approximating curves with arbitrary curvatures}
\label{subsec:duty-cycling}
As a first step towards proving Thm.~\ref{thm:decomposition}, we show that a trajectory~$\sigma$ as in Thm.~\ref{thm:decomposition}, which has arbitrary curvature, can be approximated by a finite sequence of motion primitives whose curvature is either $0$ or $\kappa_{\max}$. We provide a justification of this property below. 

When a bevel-tip needle is inserted only, it follows a trajectory with curvature $\kappa_{\max}$.
When the needle is inserted while applying axial rotational velocity that is relatively larger than the insertion velocity, it follows a straight line (i.e., of curvature zero).
Minhans et al.~\cite{Minhas2007_EMBC} introduced the notion of  duty-cycling to approximate any curvature for bevel-tip steerable needles.
Roughly speaking, combining periods of needle spinning (i.e., zero-curvature trajectories) with periods of non-spinning (i.e., maximal-curvature trajectories) enables the needle to achieve any curvature up to the maximum needle curvature.
This idea is formalized in the following lemma.
\begin{lem}[Arbitrary curvature approximation using duty-cycling]
\label{lem:duty-cycling}
    Let~$\sigma$ be a decomposable trajectory
    and let~ $\beta_d > 0$ be some real value.
    There exists a finite sequence of motion primitives $M_D$ in which every element has curvature  $\kappa \in \{0, \kappa_{\max}\}$
    such that the trajectory
    $\sigma(0) \oplus M_D$ is a piece-wise strict $\beta_d$-approximation of~$\sigma$.
\end{lem}
\begin{proofsketch}
In order to make the connection the approximation factor explicit we provide a proof from a geometric perspective (and not control-based as in the original work by Minhas et al.~\cite{Minhas2007_EMBC}).

As the trajectory $\sigma$ is decomposable, there exists a sequence of motion primitives $M_{\sigma} = \{\M_1,\dots,\M_n\}$ such that 
$\sigma = \sigma(0) \oplus M_{\sigma}$ and each motion primitive $\M_i$ has arbitrary curvature $\kappa_i \in [0, \kappa_{\max}]$.
To approximate $\M_i$, we construct a sequence of motion primitives
$M_i = \{\M_i^{(1)},\dots,\M_i^{(n_i)}\}$ that satisfies
\begin{equation*}
\begin{split}
    &\M_i^{(1)}.\delta\theta = \M_i.\delta\theta, \\
    &\forall j \in [2, n_i], \M_i^{(j)}.\delta\theta = 0, \\
    &\forall j \in [1, n_i], \M_i^{(j)}.\kappa \in \{0, \kappa_{\max}\}.
\end{split}
\end{equation*}
Namely, the first motion primitive $\M_i^{(1)}$ ensures that both trajectories use the same curving plane and the the rest of the sequence stays within this curving plane and approximates the (arbitrary) curvature $\kappa_i$.

We then decompose $\M_i$ into small equal-length segments of length $\ell_i$ (except possibly the last segment) where the specific value of $\ell_i$ is chosen according to the value of  $\beta_d $.
We then use three motion primitives to approximate each of these segments as illustrated in Fig.~\ref{fig:duty-cycling}.
Note that
(i)~the start and end configurations of $\M_i$ and $M_i$ are identical,
(ii)~the two-way Hausdorff distance between $\M_i$ and each $\M_i^{(j)}$ is less than $\beta_d' $ if $\ell_i$ is carefully chosen,
and
(iii)~for each segment with length~$\ell_i$, the length of the three-segment approximation is less than $(1 + \beta_d')\cdot \ell_i$ if $\ell_i$ is carefully chosen.
By carefully choosing $\beta_d'$, we then can make sure the three-segment approximation is a strict $\beta_d$-approximation of the original segment, where $\beta_d' \leq \beta_d$.

\begin{figure}
    \centering
    \includegraphics[width=0.8\linewidth]{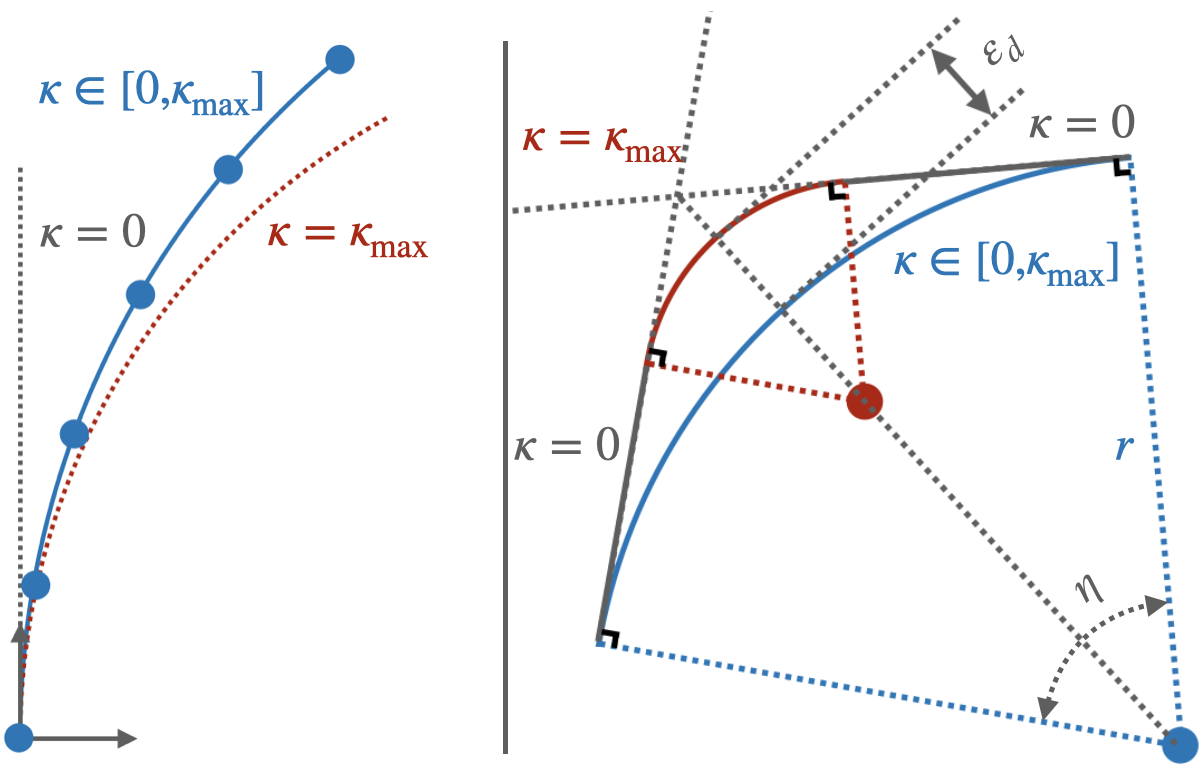}
    \caption{
    \textbf{Illustration of approximation with duty-cycling.}
    \textbf{Left:} Decompose~$\M_i$ into multiple segments with length~$\ell_i$.
    \textbf{Right:} Use three segments to approximate one segment of $\M_i$, where the segments have a curvature of $0, \kappa_{\max}$ and $0$, respectively.
    The two-way Hausdorff distance (the positional part marked as $\varepsilon_d$ in the figure) depends on $\ell_i$.
    For a given $\kappa_{\max}$, to approximate $\M_i$ (with curvature $\kappa$), the shorter $\ell_i$ is, the smaller $\varepsilon_d$ is.
    This is because
    $\varepsilon_d < r\cdot(1/{\rm cos}(0.5\eta) - 1)$,
    where $r = 1/\kappa$ is the radius of curvature and $\eta = \ell_i/r$ is the central angle.
    Since maximum orientation difference along the trajectory is bounded by 
    $0.5\eta$, the two-way Hausdorff distance in configuration space is also bounded.
    Also,
    the trajectory length of the original segment is $r\cdot\eta + \alpha\cdot\eta$,
    and the length of the three-segment trajectory is 
    $\ell_{\rm approx} \leq 2r\cdot{\rm tan}(0.5\eta) + \alpha\cdot\eta$.
    Since $\lim_{\eta \rightarrow 0} 2\cdot{\rm tan}(0.5\eta)/\eta = 1$,
    the trajectory length ratio approaches $1$ when $\eta$ approaches $0$.
    This means the trajectory length can be approximated arbitrarily well.
    }
    \label{fig:duty-cycling}
\end{figure}

Let  
$M_\sigma^{\beta_d} = M_1 \cdot M_2 \cdot \ldots \cdot M_n$ be this sequence of all the newly constructed motion primitives.
Then it is straightforward that
$\sigma(0) \oplus M_\sigma^{\beta_d}$ is a piece-wise strict $\beta_d$-approximation of $\sigma$.
\end{proofsketch}

\subsubsection{Approximating curves using fixed-resolution primitives}
\label{subsec:fixed-resolution-primitives}
Next, we refine Lem.~\ref{lem:duty-cycling} by showing $\sigma$ can be approximated by fixed-resolution primitives. 

\begin{lem}[Fixed-resolution trajectory approximation]
\label{lem:resolution-existance}
    Let~$\sigma$ be a decomposable trajectory
    and let~ $\beta_r > 0$ be some real value.
    If the system is Lipschitz continuous (Def.~\ref{dft:lipschitz}), 
    there exists a fine resolution $R(\sigma, \beta_r) = \{r_{\ell}, r_{\theta}\}$
    and a finite sequence of motion primitives 
    $M_{R(\sigma, \beta_r)}$
    such that  
    $\sigma(0) \oplus M_{R(\sigma, \beta_r)}$
    is a piece-wise strict $\beta_r$-approximation of~$\sigma$.
    Moreover $M_{R(\sigma, \beta_r)} \subseteq \M_{\rm fs}(R(\sigma, \beta_r), K_\sigma)$,
    where $K_\sigma$ is the set of curvatures that appear along $\sigma$.
\end{lem}
\begin{proofsketch}
The following is adapted from~\cite[Appendix A]{Barraquand1991_IJRR}).
The trajectory~$\sigma$ is decomposable, thus there exists a finite sequence of  motion primitives $M_\sigma = \{\M_1,\dots, \M_n\}$ such that $\sigma = \sigma(0) \oplus M_\sigma$.
Set $K_\sigma = \bigcup_i \M_i.\kappa$ to be the set of all curvatures that appear in $M_\sigma$.

To approximate each motion primitive~$\M_i$ using primitives from the finest set of motion primitives $\M_{\rm fs}(R(\sigma, \beta_r), K_\sigma)$ (Def.~\ref{dft:finest_set}), we construct a sequence motion primitive~$M_i = \{\M_i^{(1)},\dots\M_i^{(n_i)}\}$,
where 
\begin{equation*}
    \begin{split}
        &\M_i^{(1)}.\delta\theta = k_i\cdot r_{\theta},\\ &\forall j \in [2, n_i], \M_i^{(j)}.\delta\theta = 0, \\
        &\forall j \in [1, n_i], \M_i^{(j)}.\kappa = \M_i.\kappa, M_i^{(j)}.\delta\ell = r_{\ell}.
    \end{split}
\end{equation*}

Similar to the sequence constructed for Lem.~\ref{lem:duty-cycling}, the first motion primitive $\M_i^{(1)}$ accounts for the curving plane (though here it can only be approximated)
and the the rest of the sequence stays within this curving plane and accounts for the length of the circular arc the trajectory follows in this plane.
Applying the sequence~$M_i$ is equivalent to applying one motion primitive
$\tilde{\M}_i = (\M_i.\kappa, n_i\cdot r_\ell, k_i\cdot r_\theta)$.
Thus, by carefully choosing~$r_\ell$ and~$r_\theta$, 
the distance between 
$\M_i$
and
$\tilde{\M}_i$
(see Def.~\ref{dft:action-distance})
can be arbitrarily small.

This is done for every motion primitive $\M_i$.
As $M$ is a finite sequence of size $n$, for any $\varepsilon > 0$ we can always find a fine-enough resolution $\{r_\ell, r_\theta\}$ that ensures that
$$\rho_{\A}(\M_i, \tilde{\M}_i) < \varepsilon, \forall i \in [1, n].$$
\noindent
This is because, given that both motion primitives have equal curvature, 
$\rho_{\A}(\M_1, \M_2) < |\delta\theta_1 - \delta\theta_2|\cdot{\min}\{\delta\ell_1, \delta\ell_2\} + |\delta\ell_1 - \delta\ell_2|
+ \alpha\big(|\delta\theta_1 - \delta\theta_2| + \frac{|\delta\ell_1 - \delta\ell_2|}{\M_i.\kappa}\big)
$,
where
$\delta\ell_i = \M_i.\delta\ell$ and $\delta\theta_i = \M_i.\delta\theta$.
The above upper bound for the  action-space distance accounts for both position and orientation.
See Fig.~\ref{fig:action-distance} for illustration.

\begin{figure}
    \centering
    \includegraphics[width=\linewidth]{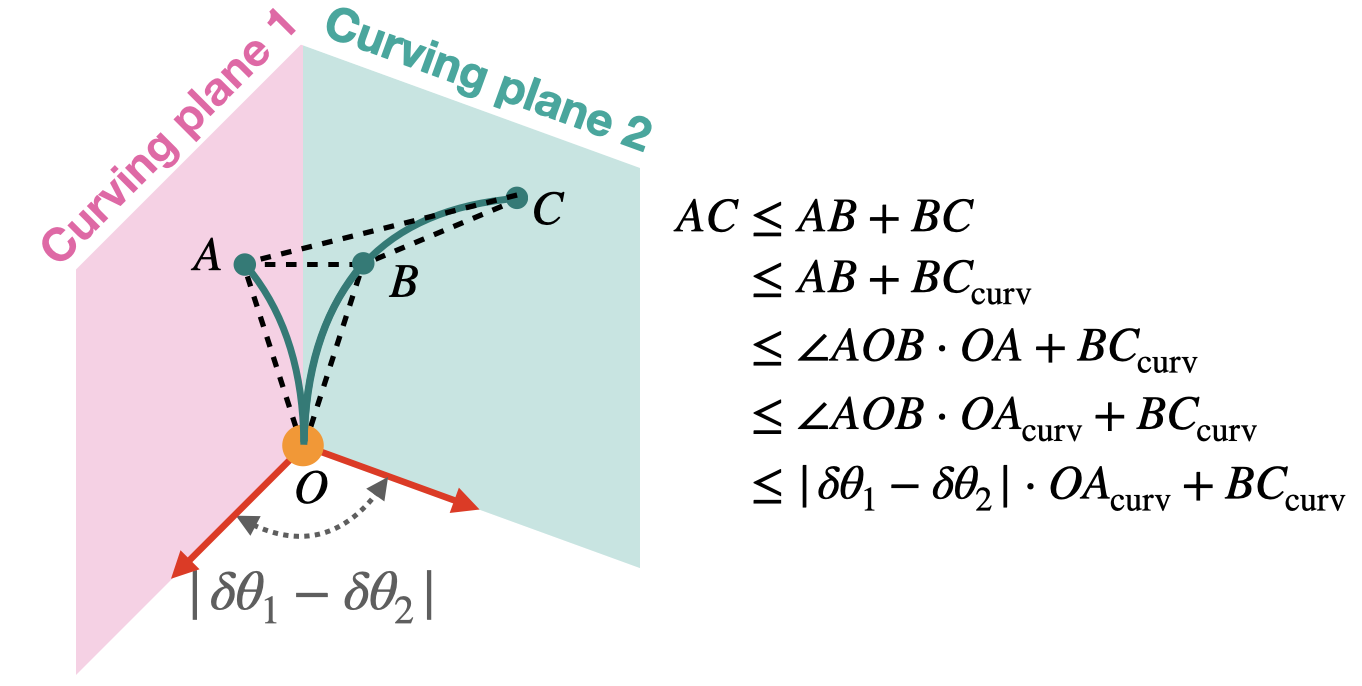}
    \caption{Illustration of the action distance between two motion primitives with the same curvature.
    Here the shorter motion primitive lies in curving plane 1, thus 
    ${\min}\{\delta\ell_1, \delta\ell_2\} = OA_{\rm curv}$
    and
    $|\delta\ell_1 - \delta\ell_2| = OC_{\rm curv} - OA_{\rm curv} = BC_{\rm curv}$.
    }
    \label{fig:action-distance}
\end{figure}

Since the system is Lipschitz continuous,
\begin{equation*}
\begin{split}
    &\rho(\sigma(0) \oplus \M_1 \dots \oplus \M_n, \sigma(0) \oplus \tilde{\M}_1 \dots \oplus \tilde{\M}_n) \\
    \leq &L_s(\rho(\sigma(0) \oplus \M_1 \dots \oplus \M_{n-1}, \sigma(0) \oplus \tilde{\M}_1 \dots \oplus \tilde{\M}_{n-1}) \\
    &+ \rho_{\A}(\M_n, \tilde{\M}_n) \\
    \leq & L_s^n\cdot \rho(\sigma(0), \sigma(0))
    + \sum_{i = 1}^{n}{L_s^{n-i+1} \cdot \rho_{\A}(\M_i, \tilde{\M}_i)}\\
    <& \varepsilon \cdot \frac{L_s(L_s^n - 1)}{L_s - 1}.
\end{split}
\end{equation*}
Thus, to ensure that $\sigma(0) \oplus \{\tilde{\M}_1, \dots, \tilde{\M}_n\}$ is a piece-wise strict 
$\beta_r$-approximation of $\sigma$, we only need to ensure that 
$\varepsilon \leq \frac{\beta_r(L_s-1)}{L_s(L_s^n - 1)}$.
As both $n$ and $L_s$ are fixed, we can choose $\varepsilon$ to be as small as needed thus the desired fine resolution exists which concludes the proof.
\end{proofsketch}

Having established Lem.~\ref{lem:resolution-existance}, we can finalize the first part of Thm.~\ref{thm:decomposition}. Namely, we carefully set $\beta_d$ and $\beta_r$, so that the final result is a piece-wise strict $\beta$-approximation.

Set $\beta_d = \beta_r = \sqrt{1 + \beta} -1$. According to Lem.~\ref{lem:duty-cycling}, 
there exists a finite sequence of motion primitives $M_D$ in which every element has curvature $\kappa \in \{0, \kappa_{\max}\}$ such that the trajectory $\sigma_d = \sigma(0) \oplus M_D$ is a piece-wise strict $\beta_d$-approximation of~$\sigma$.

Note that by construction $\sigma_d$ is decomposable.
Thus, according to Lem.~\ref{lem:resolution-existance}, 
there exists a fine resolution $R(\sigma, \beta_r) = \{r_{\ell}, r_{\theta}\}$ and a finite sequence of motion primitives  $M_{R(\sigma, \beta_r)}$ such that $\sigma_r = \sigma(0) \oplus M_{R(\sigma, \beta_r)}$ is a piece-wise strict $\beta_r$-approximation of $\sigma_d$.
Moreover, $M_{R(\sigma, \beta_r)} \subseteq \M_{\rm fs}(R(\sigma, \beta_r), \{0, \kappa_{\max}\})$ as the construction in the proof of Lem.~\ref{lem:resolution-existance} does not add new curvatures.

Finally, as $\beta_d = \beta_r = \sqrt{1 + \beta} -1$, the trajectory~$\sigma_r$ is a piece-wise strict $\beta$-approximation of $\sigma$.
This is because for every step above, we use segments of shorter lengths for the approximation, thus $\sigma_r$ and $\sigma$ satisfy $(\beta_d + \beta_r)$ point-pair-wise distance (condition (ii) and (iii)) in Def.~\ref{dft:strict-approx}.
So we only need to take care of the first condition in Def.~\ref{dft:strict-approx}.
Having $\beta_d = \beta_r = \sqrt{1 + \beta} -1$ would provide us with
$(1 + \beta_d)(1 + \beta_r) = 1 + \beta$, thus the trajectory length is also bounded, making every segments in $\sigma_r$ local strict $\beta$-approximations of the corresponding segments in $\sigma$.
By definition, $\sigma_r$ is a piece-wise strict $\beta$-approximation of $\sigma$.

\subsubsection{Similar cost for piece-wise strict approximation}
We finish Thm.~\ref{thm:decomposition} by showing that the approximation of $\sigma$ also achieves a desirable solution cost. 

\begin{lem}[Similar cost for strict approximation]
\label{lem:similar-cost}
    If a collision-free trajectory $\sigma'$ is a local strict $\beta$-approximation of another collision-free trajectory $\sigma$,
    and the cost function $\C$ is well-behaved (characterized with $L_c, c_{\min}, c_{\max}$),
    then we have
    $\C(\sigma') \leq (1 + k\cdot\beta) \cdot \C(\sigma)$,
    where $k = \frac{L_c + c_{\max}}{c_{\min}}$.
\end{lem}

\begin{proof}
We have
\begin{equation*}
    \begin{split}
        \C(\sigma') &= \int_0^{l'}c(\sigma'(s))ds \\
        & = \int_0^{l} c(\sigma'(s))ds + \int_l^{l'} c(\sigma'(s))ds \\
        & \leq \int_0^{l} \bigg( c(\sigma(s)) + L_c \cdot \beta\bigg)ds + \beta\cdot l \cdot c_{\max} \\
        & = \int_0^{l} \bigg( c(\sigma(s)) + L_c \cdot \beta + c_{\max} \cdot \beta \bigg)ds \\
        & = \int_0^{l} \bigg( 1 + \frac{\beta(L_c + c_{\max})}{c(\sigma(s))} \bigg) \cdot c(\sigma(s)) ds \\
        & \leq \int_0^{l} \bigg( 1 + \frac{\beta(L_c + c_{\max})}{c_{\min}} \bigg) \cdot c(\sigma(s)) ds \\
        & = \bigg( 1 + \frac{\beta(L_c + c_{\max})}{c_{\min}} \bigg) \cdot \int_0^{l} c(\sigma(s)) ds \\
        & = \bigg( 1 + \frac{\beta(L_c + c_{\max})}{c_{\min}} \bigg) \cdot \C(\sigma).
    \end{split}
\end{equation*}
\end{proof}

Apparently, if every piece of sub-trajectory is bounded, the sum of cost of all pieces is then also bounded.
Thus, if a trajectory $\sigma'$ is a piece-wise strict $\beta$-approximation of $\sigma$, then for a well-behaved cost, we also have
$\C(\sigma') \leq (1 + k\cdot\beta)\cdot\C(\sigma)$.

\subsection{Proof of Thm.~\ref{thm:resolution_optimal}}
We are in a position to complete the proof of Thm.~\ref{thm:resolution_optimal}. Similarly to the proof for resolution completeness~\cite{Fu2021_RSS}, we first develop the optiamlity proof for a simplified version of \ros, termed \rosnr, 
which does not use node pruning as part of duplicate detection (line~7 in Alg.~\ref{alg:shared}) and later extend it to \ros.  For simplicity, we assume that both \ros and \rosnr do not use the additional optimizations described in Sec.~\ref{sec:domain-opt} or \ref{sec:appendix_optimization}, which do not affect the validity of arguments used below.

\subsubsection{Resolution optimality of \rosnr}
Before we begin with the proof we mention that \rosnr terminates in finite time, which directly follows from our previous work~\cite[Thm. 1]{Fu2021_RSS}.

As a first step towards showing that Thm.~\ref{thm:resolution_optimal} holds for \rosnr, we consider a reference trajectory $\sigma^*$ and assume that conditions \textbf{C1}, \textbf{C2}, and \textbf{C3} are satisfied. According to \textbf{C1}, \textbf{C2}, and Thm~\ref{thm:decomposition}, as $\sigma^*$ is decomposable, for some $\beta>0$, there exists a fine resolution $R(\sigma^*, \beta)$ with which a piece-wise strict $\beta$-approximation of $\sigma^*$ can be constructed (Thm.~\ref{thm:decomposition}).
We denote such piece-wise strict $\beta$-approximation as $\sigma^*_{\beta}$. Moreover, it holds that 
$\C(\sigma_{\beta}^*) \leq (1 + k\cdot\beta)\cdot\C(\sigma^*)$.
Recall (Thm.~\ref{thm:resolution_optimal}, \textbf{C2}) that $k = \frac{L_c + c_{\max}}{c_{\min}} $.
Furthermore, we mention that the precise value of $\beta$ will be assigned later on and for now we only assume that $\beta \in (0, \frac{\delta}{2}]$.

Next, we show that $\sigma_{\beta}^*$ is valid and satisfies the desired goal tolerance, which implies that \rosnr will be able to find it. 
According to condition \textbf{C3}, $\sigma^*$ is $\delta$-robust. This implies that $\sigma_{\beta}^*$ is at least $(\delta - \beta)$-robust.
Given that $\beta \leq \frac{\delta}{2}$, we further have $\sigma_{\beta}^*$ is $\frac{\delta}{2}$-robust.
Thus for a cutoff resolution $R_{\min}$ that is fine enough, $\sigma_{\beta}^*$ will be explored by the search tree constructed by \rosnr.\footnote{To be more precise, one needs to account for the cases where $R(\sigma^*, \beta)$ is not in the sequence of resolutions considered by the algorithm and we may introduce an additional error when approximating $R(\sigma^*, \beta)$ with~$R_{\min}$. 
However, this can be easily accounted for by using a finer resolution to approximate the target resolution $R(\sigma^*, \beta)$, similarly to Lem.~\ref{lem:resolution-existance}.}

\subsubsection{Accounting for pruning}
Next, we describe how to account for pruning. 
Denote the sequence of motion primitives used to construct  $\sigma_{\beta}^*$ as 
$M_{\sigma_{\beta}^*} = \{\M_1,\dots,\M_n\}$.
When 
$M_{\sigma_{\beta}^*}$ is sequentially applied to 
$\mathbf{x}_{\rm start}$, 
we obtain a sequence of configurations
$\{\mathbf{x}_0,\mathbf{x}_1,\dots,\mathbf{x}_n\}$, 
where $\mathbf{x}_0 = \mathbf{x}_{\rm start}, \mathbf{x}_i = \mathbf{x}_{i-1} \oplus \M_i, i \in [1, n]$.
For the rest of the proof, we use $M_{\sigma_{\beta}^*}[i,j] = \{\M_i, \dots, \M_j\}$ to denote a subsequence of $M_{\sigma_{\beta}^*}$.
We also use $\mathbf{x} + M_{M_{\sigma_{\beta}^*}}[i,j]$ to denote the configuration after sequentially applying $\{\M_i, \dots, \M_j\}$ to $\mathbf{x}$.

If we run \rosnr, every configuration $\mathbf{x}_i$ will be explored and $\sigma_{\beta}^*$ will be constructed when the search terminates.
However, if we run \ros, we prune nodes using duplicate detection.
Nevertheless, we now prove that, with similar node pruning, the search tree built with \ros will explore a piece-wise strict $\frac{\delta}{2}$-approximation of $\sigma_{\beta}^*$, which implies that \ros  returns a valid solution with sufficient goal tolerance. This will be done by showing that the same sequence of motion primitives $M_{\sigma_{\beta}^*}$ can be applied to configurations that are ``similar'' to $\mathbf{x}_0 \ldots \mathbf{x}_n$ and the resultant plan $\tilde{\sigma}$ with low cost exists using the fact that $\tilde{\sigma}$ is ``similar'' to $\sigma_{\beta}^*$ and that $\sigma_{\beta}^*$ has $\frac{\delta}{2}$-clearance.
The rest of this proof formalizes this idea.

Recall that in \textbf{C4}, we require $d_{\rm sim} < \frac{2}{\kappa_{\max}}\sin{\frac{\kappa_{\max}\delta\ell_{\min}}{2}}$, which is the minimum positional difference between a node and its successor.
This condition in \textbf{C4} guarantees that any successor node is not pruned by its parent node, which keeps the tree expanding.
Now, let $\mathbf{x}_i$ be the first configuration that is pruned because of a similar configuration (see Alg.~\ref{alg:shared}, line~\ref{line:similar-node}). 
We will say that $\mathbf{x}_i$ is \emph{replaced} by the similar configuration~$\mathbf{x}'_i$.
We first consider the setting $i=1$, and apply $M_{\sigma'}[2,n]$ to $\mathbf{x}'_1$.
According to condition \textbf{C4}, the maximal error accumulated to $\mathbf{x}'_n = \mathbf{x}'_1 + M_{\sigma'}[2,n]$ is  
$\xi_1 = \rho(\mathbf{x}'_n, \mathbf{x}_n) = L_s^{n - 1}\cdot d_{\rm sim}$.
Similarly, when $\mathbf{x}'_2$ is replaced by~$\mathbf{x}''_2$, we apply $M_{\sigma'}[3,n]$ to $\mathbf{x}''_2$ and for $\mathbf{x}''_n = \mathbf{x}''_2 + M_{\sigma'}[3,n]$, the accumulated error is
$\xi_2 = \rho(\mathbf{x}''_n, \mathbf{x}'_n) = L_s^{n - 2}\cdot d_{\rm res}$.
The same analysis applies for $\{\mathbf{x}_3,\dots,\mathbf{x}_n\}$.
In the worst case, $\x_n$ can be replaced $n$ times, which leads to the total accumulated error of
\begin{equation*}
\begin{split}
    \xi 
    &= \rho(\mathbf{x}^{(n)}_n, \mathbf{x}_n) 
    \leq \rho(\mathbf{x}'_1, \mathbf{x}_1) + \dots + \rho(\mathbf{x}^{(n)}_n, \mathbf{x}^{(n-1)}_n)   \\
    &= \xi_1 + \dots + \xi_n 
    = \frac{L_s^n - 1}{L_s -1}\cdot d_{\rm sim} \\
    &< \frac{\delta}{2} \cdot \frac{L_s^n - 1}{L_s^H -1} \leq \frac{\delta}{2}.
\end{split}
\end{equation*}
 
Next we show that the solution $\tilde{\sigma}$ represented by the sequence of configurations $\x_0,\x'_1,\ldots,\x_n^{(n)}$ satisfied goal tolerance (we show that $\tilde{\sigma}$ is valid after bounding its below). Indeed, due to the robustness of $\sigma_{\beta}^*$, we have that $\|{\rm Prog}(\mathbf{x}_n) - g_{\rm goal}\|_2 < \tau - \frac{\delta}{2}$. Thus,
\begin{equation*}
\begin{split}
    &\|{\rm Proj}(\mathbf{x}^{(n)}_n) - p_{\rm goal}\|_2 \\
    &\leq \|{\rm Proj}(\mathbf{x}^{(n)}_n) - {\rm Proj}(\mathbf{x}_n)\|_2 + \|{\rm Proj}(\mathbf{x}_n) - p_{\rm goal}\|_2 \\
    &< \tau - \frac{\delta}{2} + \frac{\delta}{2} = \tau.
\end{split}
\end{equation*}
This implies that even in the worst case where all possible replacements happen, the final configuration $\mathbf{x}^{(n)}_n$ still satisfies the required goal tolerance (see Fig.~\ref{fig:completeness}).

\subsubsection{Incorporating solution cost and proving validity}
Next, we consider the cost of $\tilde{\sigma}$.
Note that we only allow $\mathbf{x}_i'$ to prune $\mathbf{x}_i$ when 
$\C(\mathbf{x}_i^{(i)}) \leq \C(\mathbf{x}_i^{(i-1)})$.
Thus, for the final configuration along $\tilde{\sigma}$ we have 
\begin{equation*}
\begin{split}
    \C(\mathbf{x}_{n}^{(n)})
    & \leq \C(\mathbf{x}_{n-1}^{(n-1)}) + \C(\mathbf{x}_{n-1}^{(n-1)}, \mathbf{x}_{n}^{(n)})\\
    & \leq \sum_{i=1}^n {\C(\mathbf{x}_{i-1}^{(i-1)}, \mathbf{x}_{i}^{(i)})}.
\end{split}
\end{equation*}

It remains to bound the expression $\C(\mathbf{x}_{i-1}^{(i-1)}, \mathbf{x}_{i}^{(i)})$ for any $1\leq i\leq n$. 
Denote by $\sigma^*_\beta(\x_{i-1},\x_i)$ the trajectory segment from $\x_{i-1}$ and $\x_i$ along $\sigma^*_\beta$. Similarly define $\tilde{\sigma}(\x_{i-1}^{(i-1)},\x_i^{(i)})$.
We now show that $\tilde{\sigma}(\x_{i-1}^{(i-1)},\x_i^{(i)})$ is local $\beta_i$-strict approximation of $\sigma^*_\beta(\x_{i-1},\x_i)$
for 
$\beta_i = \frac{L_s^i - 1}{L_s - 1}\cdot d_{\rm sim}$.
To see that, first note that both trajectory segments use the same motion primitive and have the same length. Additionally, 
$\rho(\mathbf{x}_i^{(i)}, \mathbf{x}_i) \leq \frac{L_s^i - 1}{L_s - 1}\cdot d_{\rm sim}$.
Finally, An intermediate state $\mathbf{x}'$ along the edge is also close to the corresponding state $\mathbf{x}$ along the original edge, since they can be obtained by applying a motion primitive of a shorter length, and Lipshitz continuity of the system guarantees
$\rho(\mathbf{x}', \mathbf{x}) \leq \frac{L_s^i - 1}{L_s - 1}\cdot d_{\rm sim}$. Thus, since $\beta_i\leq \beta_n$,
these local $\beta_i$-strict approximations are also local $\beta_n$-strict approximations.

According to the similar-cost property of local strict approximation, we have that 
\begin{equation*}
    \C(\mathbf{x}_{i-1}^{(i-1)}, \mathbf{x}_{i}^{(i)}) \leq (1 + k\beta_n)\cdot\C(\mathbf{x}_{i-1}, \mathbf{x}_{i}),
\end{equation*}
where $k$ is as defined in \textbf{C2}.
To summarize, the accumulated cost of $\tilde{\sigma}$ is as follows:
\begin{equation*}
\begin{split}
    \C(\mathbf{x}_{n}^{(n)})
    & \leq \sum_{i=1}^n {\C(\mathbf{x}_{i-1}^{(i-1)}, \mathbf{x}_{i}^{(i)})}\\
    & \leq (1 + k\beta_n)\cdot\sum_{i=1}^n {\C(\mathbf{x}_{i-1}, \mathbf{x}_{i})} \\
    & = (1 + k\beta_n)\cdot \C(\sigma_{\beta}^*) \\
    & \leq \left(1 + k\cdot\frac{\delta}{2}\right)\cdot \C(\sigma_{\beta}^*) \\
    & \leq \left(1 + k\cdot\frac{\delta}{2}\right) (1 + k\beta)\cdot \C(\sigma^*) \\
    &= \left(1 + k\cdot\frac{\delta}{2} + k\cdot\beta + k^2\cdot\frac{\delta\cdot\beta}{2}\right)\cdot\C(\sigma^*).
\end{split}
\end{equation*}

\begin{figure}
    \centering
    \includegraphics[width=\linewidth]{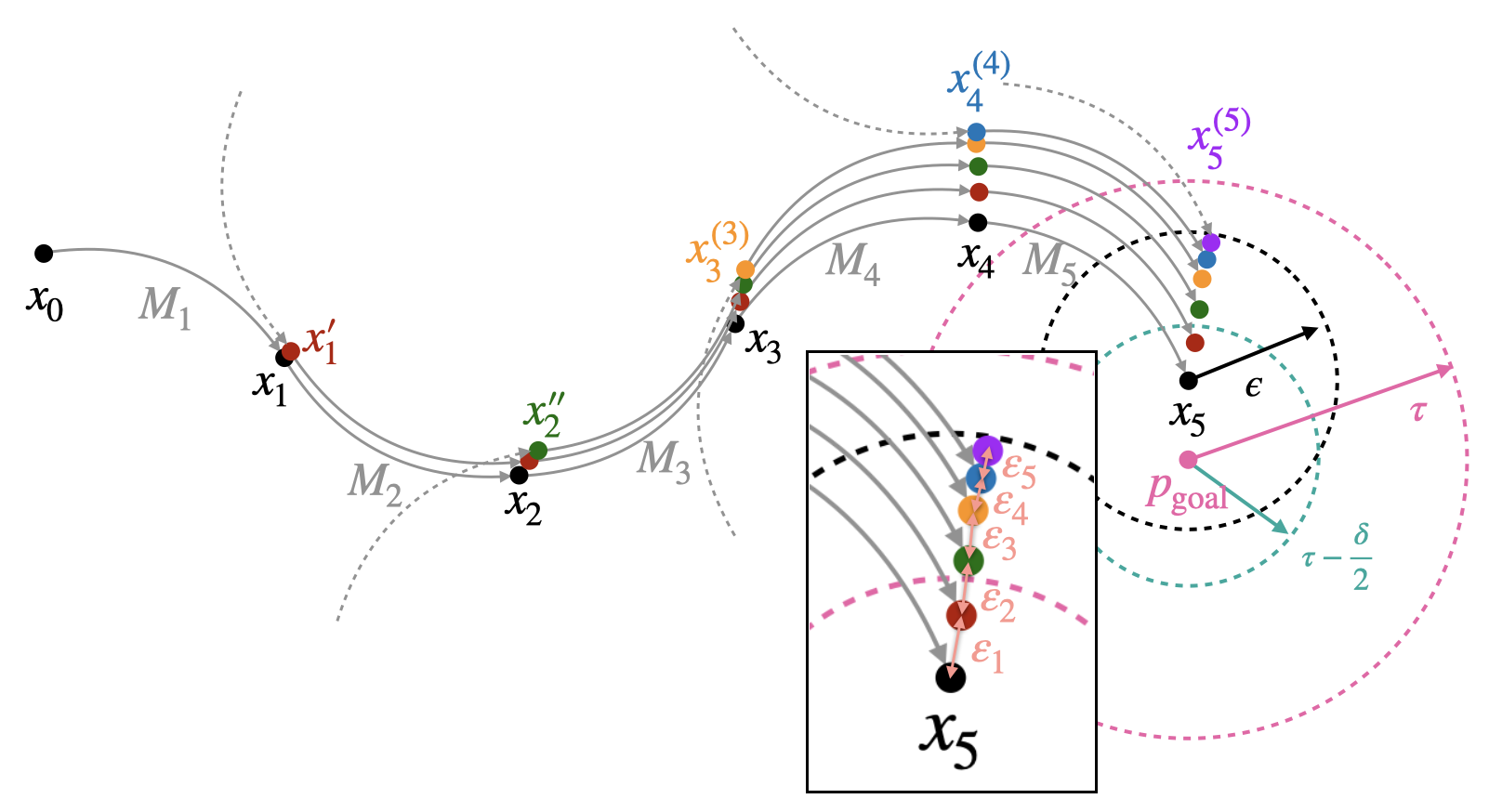}
    \caption{A 2D illustration of configuration pruning. 
    $\sigma_{\beta}^*$ is shown as black nodes, 
    the plan after~$\mathbf{x}'_1$ prunes~$\mathbf{x}_1$ is shown as red nodes,
    the plan after~$\mathbf{x}''_2$ prunes~$\mathbf{x}'_2$ is shown as green nodes,
    the plan after~$\mathbf{x}^{(3)}_3$ prunes~$\mathbf{x}''_3$ is shown as yellow nodes,
    the plan after~$\mathbf{x}^{(4)}_4$ prunes~$\mathbf{x}^{(3)}_4$ is shown as blue nodes,
    and the pruning configuration $\mathbf{x}^{(5)}_5$ is shown as a purple node.
    The solid circular arrows represent elements in $M_{\sigma_{\beta}^*}$, and the dashed circular arrows represent connections to predecessors of the pruning configurations.
    }
    \label{fig:completeness}
\end{figure}

It remains to determine the value of $\beta$ to achieve a desired approximation factor of $1+\varepsilon$. So far we required $\beta \leq \frac{\delta}{2}$.
To further guarantee  that
$\C(\mathbf{x}_n^{(n)}) \leq (1 + \varepsilon)\cdot\C(\sigma^*)$ holds, we further require that
$\beta \leq \frac{2\varepsilon - k\cdot\delta}{k(2 + k\cdot\delta)}$.
Thus we take
$\beta = {\min}\{\frac{\delta}{2}, \frac{2\varepsilon - k\cdot\delta}{k(2 + k\cdot\delta)}\}$.
According to condition \textbf{C3} it follows that, $\delta \leq \frac{\varepsilon}{k}$, so we always have $\beta > 0$.

It remains to show that $\tilde{\sigma}$ is valid. 
We have shown above that 
$\tilde{\sigma}(\x_{i-1}^{(i-1)},\x_i^{(i)})$ is local $\beta_n$-strict approximation of $\sigma^*_\beta(\x_{i-1},\x_i)$. 
This implies that for any configuration~$\mathbf{x}'$ along $\tilde{\sigma}(\x_{i-1}^{(i-1)},\x_i^{(i)})$ there exists some corresponding configuration~$\mathbf{x}$ along $\sigma^*_\beta(\x_{i-1},\x_i)$ such that
$\rho(\mathbf{x}', \mathbf{x}) \leq \beta_n \leq \frac{\delta}{2}$.
Due to the $\sigma^*_\beta$ being $\frac{\delta}{2}$-robust, we are guaranteed that the motion plan $\tilde{\sigma}$ is collision free.

To summarize, as long as the required conditions are satisfied, \ros still finds a valid motion plan $\sigma$ that satisfies $\C(\sigma) \leq (1 + \varepsilon)\cdot\C(\sigma^*)$.

\end{document}